\documentclass[Afour,sageh,times]{sagej}

\usepackage{moreverb,url}
\usepackage[colorlinks,bookmarksopen,bookmarksnumbered,citecolor=blue,urlcolor=blue]{hyperref}

\usepackage{subcaption}
\usepackage{xspace}
\usepackage{enumitem}
\usepackage{amsmath,amssymb,amsfonts,amsthm}
\usepackage{algorithm,algorithmicx}
\usepackage[noend]{algpseudocode}
\algrenewcommand\algorithmicindent{0.5em} 
\usepackage{balance}
\usepackage{adjustbox}
\usepackage{soul}

\graphicspath{
    {./figs/}
}

\newcommand{\calG}{\ensuremath{\mathcal{G}}\xspace}

\newcommand{\calM}{\ensuremath{\mathcal{M}}\xspace}

\newcommand{\calP}{\ensuremath{\mathcal{P}}\xspace}

\newcommand{\calO}{\ensuremath{\mathcal{O}}\xspace}

\def\naive{{na\"{\i}ve}\xspace}

\newtheorem{thm}{Theorem}

\newtheorem{definition}[thm]{Definition}
\newtheorem{lemma}{Lemma}

\DeclareMathOperator*{\argmin}{arg\,min}

\newcommand{\ignore}[1]{}

\newcommand\algname[1]{\textsf{#1}\xspace}

\newcommand\Gfull{\ensuremath{G^{\textrm{full}}}\xspace}
\newcommand\Gcov{\ensuremath{G^{\textrm{cov}}}\xspace}
\newcommand\Guncov{\ensuremath{G^{\textrm{uncov}}}\xspace}

\newcommand\Guncovp{\ensuremath{G'^{\textrm{uncov}}}\xspace}
\newcommand\Shome{\ensuremath{s_{\textrm{home}}}\xspace}
\newcommand\Tbound{\ensuremath{T_{\textrm{bound}}}\xspace}
\newcommand\Tconst{\ensuremath{T_{\textrm{const}}}\xspace}
\newcommand\Trc{\ensuremath{t_{\textrm{rc}}}\xspace}
\newcommand\Ssc{\ensuremath{s_{\textrm{sc}}}\xspace}
\newcommand\Sstart{\ensuremath{s_{\textrm{start}}}\xspace}
\newcommand\Xexec{\ensuremath{x_{\textrm{exec}}}\xspace}

\def\os#1{\textcolor{magenta}{#1}}

\def\AlgFontSize{\small}
\def\CaptionTextSize{\small}

\newcommand\BibTeX{{\rmfamily B\kern-.05em \textsc{i\kern-.025em b}\kern-.08em
T\kern-.1667em\lower.7ex\hbox{E}\kern-.125emX}}

\def\volumeyear{2020}

\begin{document}

\runninghead{Islam et al.}

\title{Provably Constant-time Planning and Replanning for Real-time Grasping Objects off a Conveyor Belt}

\author{Fahad Islam\affilnum{1} 
        and
        Oren Salzman\affilnum{2}
        and 
        Aditya Agarwal\affilnum{1}
        and
        Maxim Likhachev\affilnum{1}}

\affiliation{%
\affilnum{1}The Robotics Institute, Carnegie Mellon University\\
\affilnum{2}Technion-Israel Institute of Technology}

\corrauth{Fahad Islam.}

\email{fahad.islam@fulbrightmail.org}

\begin{abstract}
In warehouse and manufacturing environments, manipulation platforms are frequently deployed at conveyor belts to perform pick and place tasks. Because objects on the conveyor belts are moving, robots have limited time to pick them up. This brings the requirement for fast and reliable motion planners that could provide provable real-time planning guarantees, which the existing algorithms do not provide. Besides the planning efficiency, the success of manipulation tasks relies heavily on the accuracy of the perception system which is often noisy, especially if the target objects are perceived from a distance. For fast moving conveyor belts, the robot cannot wait for a perfect estimate before it starts executing its motion. In order to be able to reach the object in time, it must start moving early on (relying on the initial noisy estimates) and adjust its motion on-the-fly in response to the pose updates from perception. We propose a planning framework that meets these requirements by providing provable constant-time planning and replanning guarantees.
To this end, we first introduce and formalize a new class of algorithms called \emph{Constant-Time Motion Planning algorithms (CTMP)} that guarantee to plan in constant time and within a user-defined time bound. We then present our planning framework for grasping objects off a conveyor belt as an instance of the CTMP class of algorithms.
We present it, give its analytical properties and show experimental analysis both in simulation and on a real robot.
\end{abstract}


\maketitle

\section{Introduction}
Conveyor belts are widely used in automated distribution, warehousing, as well as for manufacturing and production facilities. In the modern times robotic manipulators are being deployed extensively at the conveyor belts for automation and faster operations~\cite{zhang2018gilbreth}. In order to maintain a high-distribution throughput, manipulators must pick up moving objects without having to stop the conveyor for every grasp. In this work, we consider the problem of motion planning for grasping moving objects off a conveyor. An object in motion imposes a requirement that it should be picked up in a short window of time. The motion planner for the arm, therefore, must compute a path within a bounded time frame to be able to successfully perform this task.

\begin{figure}[t]
    \centering
    \includegraphics[trim=0 50 0 100, clip, width=\columnwidth]{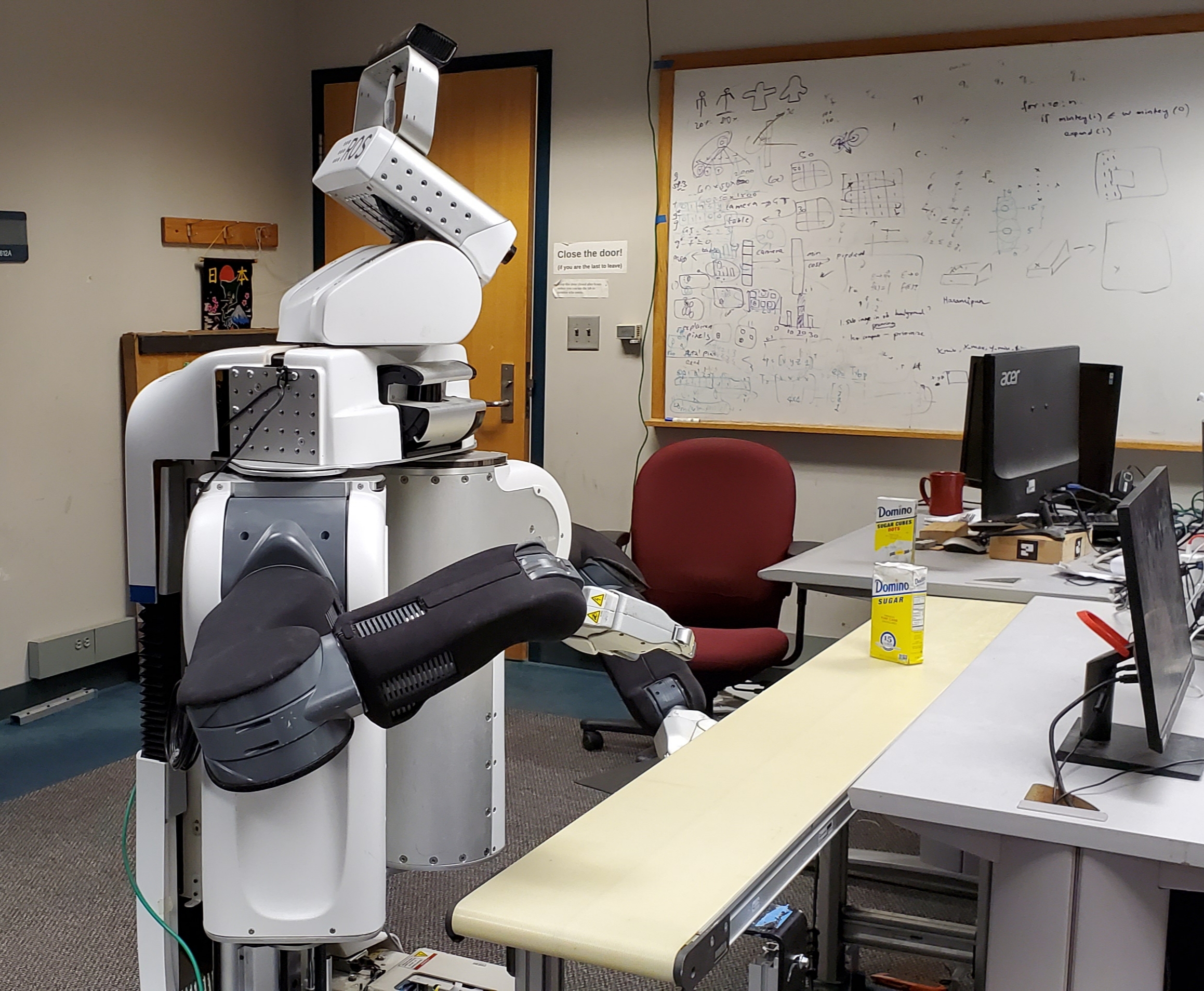}
    \caption{
    \CaptionTextSize
    A scene demonstrating the PR2 robot picking up a moving object (sugar box) off a conveyor belt.}
    \label{fig:intro_pic}
\end{figure}

Manipulation relies on high quality detection and localization of moving objects. When the object first enters the robot's field of view, the initial perception estimates of the object's pose are often inaccurate. Consider the example of an object (sugar box) moving along the conveyor towards the robot in Fig.~\ref{fig:intro_pic}, shown through an image sequence as captured by the robot's Kinect camera in Fig.~\ref{fig:pose_sequence}. 
The plot in Fig.~\ref{fig:pose_sequence} shows the variation of the error between the filtered input point cloud and a point cloud computed from the predicted pose from our Iterative Closest Point (ICP)-based perception strategy~\cite{ISL19} as the object gets closer to the camera. We observe that the error decreases as the object moves closer, indicating that the point clouds overlap more closely due to more accurate pose estimates closer to the camera.

However, if the robot waits too long to get an accurate estimate of the object pose, the delay in starting plan execution could cause the robot to miss the object. The likelihood of this occurring as the speed of the conveyor increases. Therefore, the robot should start executing a plan computed for the initial pose and as it gets better estimates, it should repeatedly replan for the new goals. However, for every replanning query, the time window for the pickup shrinks. This makes the planner's job difficult to support real-time planning.

Furthermore, the planning problem is challenging because the motion planner has to account for the dynamic object and thus plan with time as one of the planning dimensions. It should generate a valid trajectory that avoids collision with the environment around it and also with the target object to ensure that it does not damage or topple it during the grasp. Avoiding collisions with the object requires precise geometric collision checking between the object geometry and the geometry of the manipulator.
The robot arms also have kinodynamic constraints such as torque and velocity limits that the motion planner may have to account for while computing the plans, especially when the robots must move at high speeds.
The resulting complexity of the planning problem makes it infeasible to plan online for this task.

Motivated by these challenges, we propose an algorithm that leverages offline preprocessing to provide bounds on the planning time when the planner is invoked online. Our key insight is that in our domain the manipulation task is highly repetitive. Even for different object poses, the computed paths are quite similar and can be efficiently reused to speed up online planning. Based on this insight, we derive a method that precomputes a representative set of paths with some auxiliary datastructures offline and uses them online in a way that provides \emph{constant-time planning guarantee}.
We present it as an instance of a new class of algorithms which we call Constant-time Motion Planning (CTMP) algorithms.
To the best of our knowledge, our approach is the first to provide constant-time planning guarantee on generating motions for indefinite horizons i.e. all the way to the goal.

%
We experimentally show that constant-time planning and replanning capability is necessary for a successful conveyor pickup task. Specifically if we only perform one-time planning, that is planning only once either for the very first and potentially inaccurate pose estimate or for a delayed but accurate pose estimate of the object, the robot frequently fails to pick the object.
\begin{figure}[t]
    \centering
    \begin{subfigure}{.23\textwidth}
        \includegraphics[height=2.9cm]{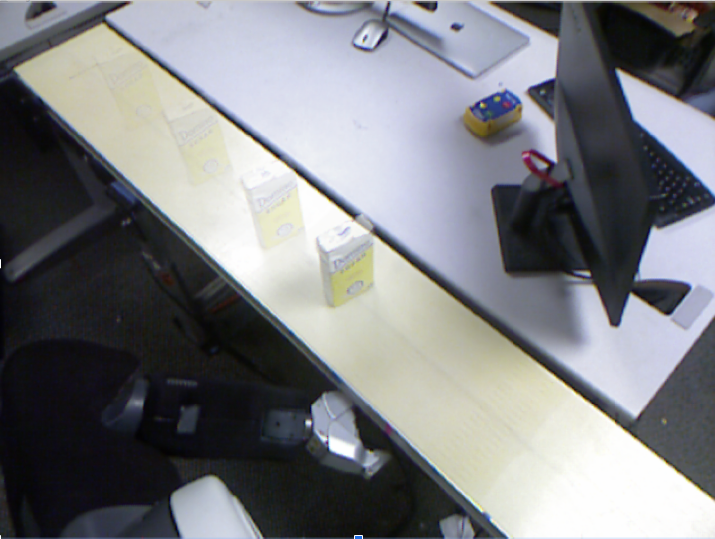}
        \caption{}
        \label{fig:obj1}
    \end{subfigure}
    \begin{subfigure}{0.25\textwidth}
        \includegraphics[height=2.9cm]{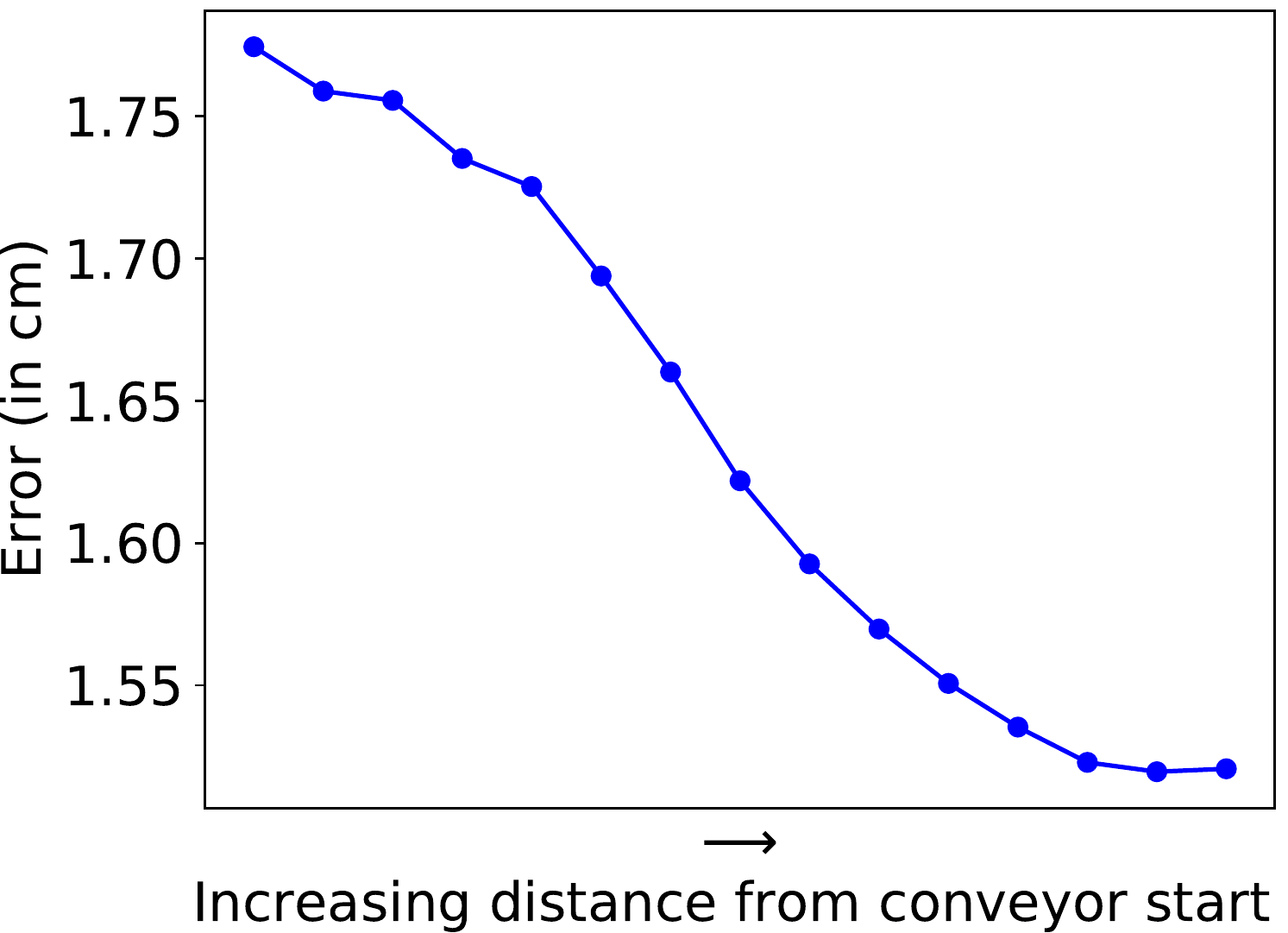}
        \caption{}
        \label{fig:obj2}
    \end{subfigure}
    \caption{
    \CaptionTextSize
    (\subref{fig:obj1})~Depiction of an object moving along a conveyor towards the robot.
    (\subref{fig:obj2})~Pose error as a function of the distance from the conveyor's start. Specifically we use ADD-S error \cite{add_metric}.
    }
    \label{fig:pose_sequence}
\end{figure}
\subsection{Statement of Contributions}
We make the following contributions in this work:
\begin{enumerate}[noitemsep]
    \item We develop a provably constant-time planning and replanning framework for grasping moving objects off a conveyor belt.
    \item We prove that the algorithm is complete and is guaranteed to run in constant time and within a user specified time bound.
    \item We provide experimental analysis of our approach in simulation as well as on a physical robotic system.
    \item \label{c1} We introduce and formalize a new class of algorithms called Constant-Time Motion Planning (CTMP) algorithms and show that the proposed approach for grasping objects off the conveyor is an instance of CTMP class of algorithms.
    \item \label{c2} We develop a kinodynamic motion planner to account for the dynamic constraints of the robot including joint torque and velocity limits.
\end{enumerate}
This article is in continuation of our previous work presented in~\cite{islamprovably} and the contributions~\ref{c1} and~\ref{c2} specifically, are the extensions. In addition to these extensions, we provide space complexity analysis of our approach and report detailed preprocessing statistics of our experiments to highlight the improvement over the brute force method. We also remove one of the assumptions of~\cite{islamprovably} which says that the environment remains static up to a certain time in execution.

\section{Related work}
\subsection{Motion planning for conveyor pickup task}
Existing work on picking moving objects has focused on different aspects of the problem ranging from closed-loop controls, servoing to object perception and pose estimation, motion planning and others~\cite{zhang2018gilbreth,allen1993automated, han2019toward, stogl2017tracking}. 
Here, we focus on motion-planning related work. Time-configuration space representation was introduced to avoid moving obstacles~\cite{fraichard1993dynamic,cefalo2013task,yang2018planning}. Specifically in ~\cite{yang2018planning}, a bidirectional sampling-based method with a time-configuration space representation was used to plan motions in dynamic environments to pickup moving objects. While their method showed real-time performance in complex tasks, it used fully-specified goals; namely knowing the time at which the object should be picked, which weakens the completeness guarantee. Furthermore their method is probablistically complete and therefore, does not offer constant-time behavior.
Graph-search based approaches have also been used for the motion-planning problem~\cite{menon2014motion, cowley2013perception}. The former uses a kinodynamic motion planner to smoothly pick up moving objects i.e., without an impactful contact. A heuristic search-based motion planner that plans with dynamics and could generate high-quality trajectories with respect to the time of execution was used. While this planner provides strong guarantees on the solution quality, it is not real-time and thus cannot be used online.
The latter work demonstrated online real-time planning capability. The approach plans to a pregrasp pose with pure kinematic planning and relies on Cartesian-space controllers to perform the pick up. The usage of the Cartesian controller limits the types of objects that the robot can grasp. It also provided no guarantee to be able to generate a plan in time to execute.

\subsection{Preprocessing-based planning}
Preprocessing-based motion planners often prove beneficial for real-time planning. They analyse the configuration space offline to generate some auxiliary information that can be used online to speed up planning. 
Probably the best-known example is the Probablistic Roadmap Method (PRM)~\cite{kavraki1996probabilistic} which precomputes a roadmap that can answer any query by connecting the start and goal configurations to the roadmap and then searching the roadmap. PRMs are fast to query yet they do not provide constant-time guarantees.
To account for dynamic environments, several PRM method-based extensions have been proposed~\cite{yoshida2011reactive,belghith2006anytime}. However, these methods often require computationally expensive repairing operations which cause additional overheads. 
%

A provably constant-time planner was recently proposed in~\cite{ISL19}. Given a start state and a goal region, it precomputes a compressed set of paths that can be utilised online to plan to any goal within the goal region in bounded time. As we will see, our approach while algorithmically different, draws some inspiration from this work.
Both of the above two methods (\cite{kavraki1996probabilistic,ISL19}) mainly target pure kinematic planning and thus they cannot be used for the conveyor-planning problem which is dynamic in nature.

Another family of preprocessing-based planners utilises previous experiences to speed up the search~\cite{BAG12,CSMOC15,PCCL12}. Experience graphs~\cite{PCCL12}, provide speed up in planning times for repetitive tasks by trying to reuse previous experiences. These methods are also augmented with sparsification techniques (see e.g.,~\cite{DB14,SSAH14}) to reduce the memory footprint of the algorithm.
Unfortunately, none of the mentioned algorithms provide fixed planning-time guarantees that we strive for in our application.

\subsection{Online replanning and real time planning}
The conveyor-planning problem can be modelled as a Moving Target Search problem (MTS) which is a widely-studied topic in the graph search-based planning literature~\cite{ishida1991moving,ishida1995moving,koenig2007speeding,sun2010moving}. 
These approaches interleave planning and execution incrementally and update the heuristic values of the state space to improve the distance estimates to the moving target. Unfortunately, in high-dimensional planning problems, this process is computationally expensive which is why these approaches are typically used for two-dimensional grid problem such as those encountered in video games. More generally, real-time planning is widely considered in the search community (see, e.g.,~\cite{KL06,KS09,korf1990real}).
However, as mentioned, these works are typically applicable to low-dimensional search spaces.

\section{Constant-Time Motion Planning (CTMP) for Indefinite Horizons}
Before we formally define our motion-planning problem, we introduce a new \emph{class} of algorithms that we call \emph{Constant-Time Motion Planning algorithms (CTMP)} that are specially designed for repetitive robotic tasks. We will see that the algorithm we propose for real-time grasping objects off a conveyor belt is an instance of a CTMP algorithm. However, we believe that their definition is of independent interest regardless of the specific problem instance that is considered. 
While limited-horizon planners like real-time heuristic search-based planners also plan in constant-time, this paper specifically addresses the indefinite horizon planning problem, that is to plan all the way to the goal~\cite{KL06,KS09,korf1990real}.

\subsection{CTMP Problem Definition}
We consider the problem setting where the robot has a fixed start state~$\Sstart$ and a set of goals~$G$, the representation and specification of which are domain dependent.
The algorithm is queried for a pair ($\Sstart, g$) where $g \in G$ and attempts to compute a feasible path from \Sstart to $g$ within a (small) constant time.



\begin{definition}[\textbf{CTMP Algorithm}]
\label{ctmp:def}
Let \algname{ALG} be a motion-planning algorithm,
$\Tbound$ a user-controlled time bound
and
$\Tconst < \Tbound$ a small time constant whose value is independent of the size and complexity of the motion-planning problem.
\algname{ALG} is said to be a \emph{CTMP algorithm} if it is guaranteed to answer any motion planning query within~\Tbound.
\end{definition}
The input time~$\Tbound$ is a tunable parameter that can be used to trade off query time with the memory footprint and preprocessing efforts of the algorithm.
It is important to note that an algorithm that returns NO$\_$PATH for any query is CTMP. Thus, unless endowed with additional properties CTMP is a weak algorithmic property.
We now define one such property, namely CTMP-Completeness, that makes a CTMP algorithm interesting and useful.
\subsection{CTMP-Completeness}
We assume a CTMP algorithm has access to a regular motion-planning algorithm~$\calP$ (that is not necessarily CTMP).
We introduce a new notion of completeness for CTMP algorithms. Prior to that, we need to define some preliminaries.

\begin{definition}[Reachability]
\label{def:reachable}
A goal $g \in G$ is said to be reachable from a state \Sstart if~$~\calP$ can find a path to it within a (sufficiently large) time bound~$T_\textrm{\calP}$. 
\end{definition}

\begin{definition}[Goal Coverage]
\label{def:covered}
    A reachable goal $g \in G$ is said to be covered by the CTMP algorithm for a state \Sstart if (in the query phase) it can plan from \Sstart to $g$ while satisfying Def.~\ref{ctmp:def}.
\end{definition}

We are now equipped to define CTMP-Completeness.
\begin{definition}[\textbf{CTMP-Completeness}]
\label{def:complete}
    An algorithm is said to be CTMP-complete if it covers all the reachable goals $g \in G$ for \Sstart.
\end{definition}
A CTMP-algorithm provides a planning time bound reduction of ~$T_\calP : \Tbound$, where ~$\Tbound \ll T_\textrm{\calP}$, while still guaranteeing a success rate no worse than what \calP would provide with a time bound of $T_\calP$.
A CTMP-complete algorithm distinguishes from a CTMP algorithm for it must \emph{cover} all \emph{reachable} goals in $G$ for \Sstart in the preprocessing phase, so that in the query phase, it can find a plan to any reachable goal within \Tbound time. A CTMP algorithm which is not CTMP-complete however, may return failure even if the queried goal is reachable.

Note that the notion of CTMP-completeness is decoupled from the completeness guarantee (by conventional definition) of the underlying motion planner \calP, hence we can view a CTMP algorithm as a meta planner. The CTMP algorithm only provides guarantees for the subset of $G$ which is reachable and hence its completeness properties differ from the completeness properties of~\calP. However, the size of the reachable set of goals is largely dependent on the performance of \calP.

\section{CTMP for Grasping Objects off A Conveyor Belt---Problem Setup}
Our system is comprised of 
a robot manipulator,
a conveyor belt moving at some known velocity,
a set of known objects~$\calO$ that need to be grasped and 
a perception system that is able to estimate the type of object and its location on the conveyor belt. Here, we assume that the geometric models of the target objects are known apriori and we are given a set of feasible grasps for each object~$o \in \calO$.

Given a pose $g$ of an object $o \in \calO$, our task is to plan the motion of the robot such that it grasps~$o$ from the conveyor belt at some future time.
Unfortunately, the perception system may give inaccurate object poses.
Thus, the pose $g$ will be updated by the perception system as the robot is executing its motion. 
To allow for the robot to move towards the updated pose in real time, we introduce the additional requirement that planning should be done within a user-specified time bound~\Tbound.
For ease of exposition, when we say that we plan to a pose $g$ of $o$ that is given by the perception system, 
we mean that we plan the motion of the robot such that it will be able to pick~$o$ from the conveyor belt at some future time. 
This is explained in detail in Sec.~\ref{sec:robot_results} and in Fig.~\ref{fig:pe}.

We denote by $\Gfull$ the discrete set of initial object poses on the conveyor belt that the perception system can perceive.
Finally, we assume that the robot has an initial state \Shome corresponding to the time $t=0$ from which it starts planning to grasp any object.

Roughly speaking, the objective is to enable planning and replanning to any goal pose $ g \in \Gfull$ in bounded time~\Tbound regardless of the robot's current state.
%
%
%
%
Specifically we aim to develop a CTMP-complete algorithm, so that 
for any state $s$ that the system can be in 
and every reachable goal pose $g \in \Gfull$ from $s$ updated by the perception system,
$g$ is covered by $s$.

We assume that the pose estimation error of the perception system is bounded by a distance~$\varepsilon$.
We also specify a replan cutoff time $t=\Trc$ after which the planner does not replan and continues to execute the last planned path until the goal is reached. The time~$\Trc$ could be chosen based on when the perception system is expected to send an accurate pose estimate, and therefore no replanning is needed from there on.

\section{Algorithmic framework}
\label{subsec:strawman}
Our approach for constant-time planning relies on a \emph{preprocessing} stage that allows to efficiently compute paths in a \emph{query} stage to any goal.
%
Before we describe our approach, we start by describing a \naive method that solves the aforementioned problem but requires a prohibitive amount of memory.
This can be seen as a warmup before describing our algorithm which exhibits the same traits but does so in a memory-efficient manner.

\subsection{Straw man approach}
\begin{figure*}[t]
    \centering
    \begin{subfigure}{.35\textwidth}
        \includegraphics[width=0.9\textwidth]{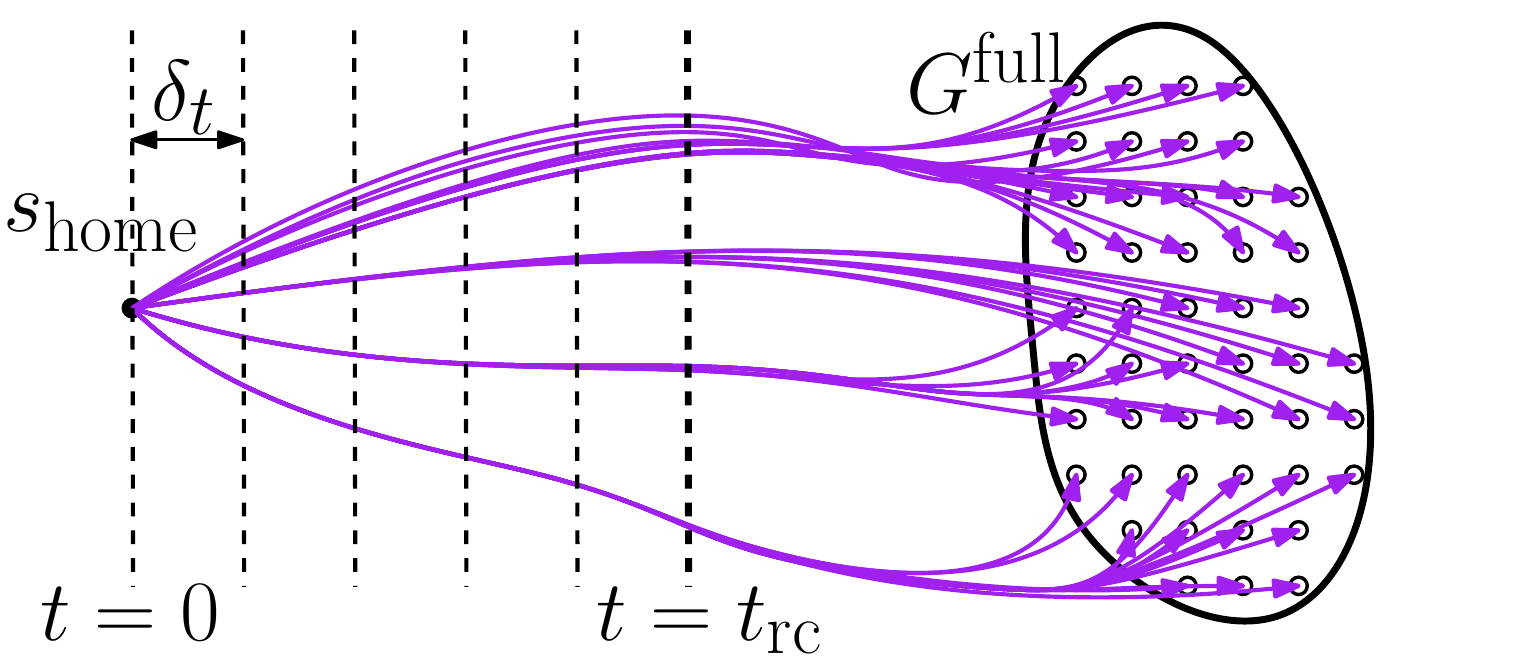}
        \caption{}
        \label{fig:naive1}
    \end{subfigure}
    \begin{subfigure}{0.35\textwidth}
        \includegraphics[width=0.9\textwidth]{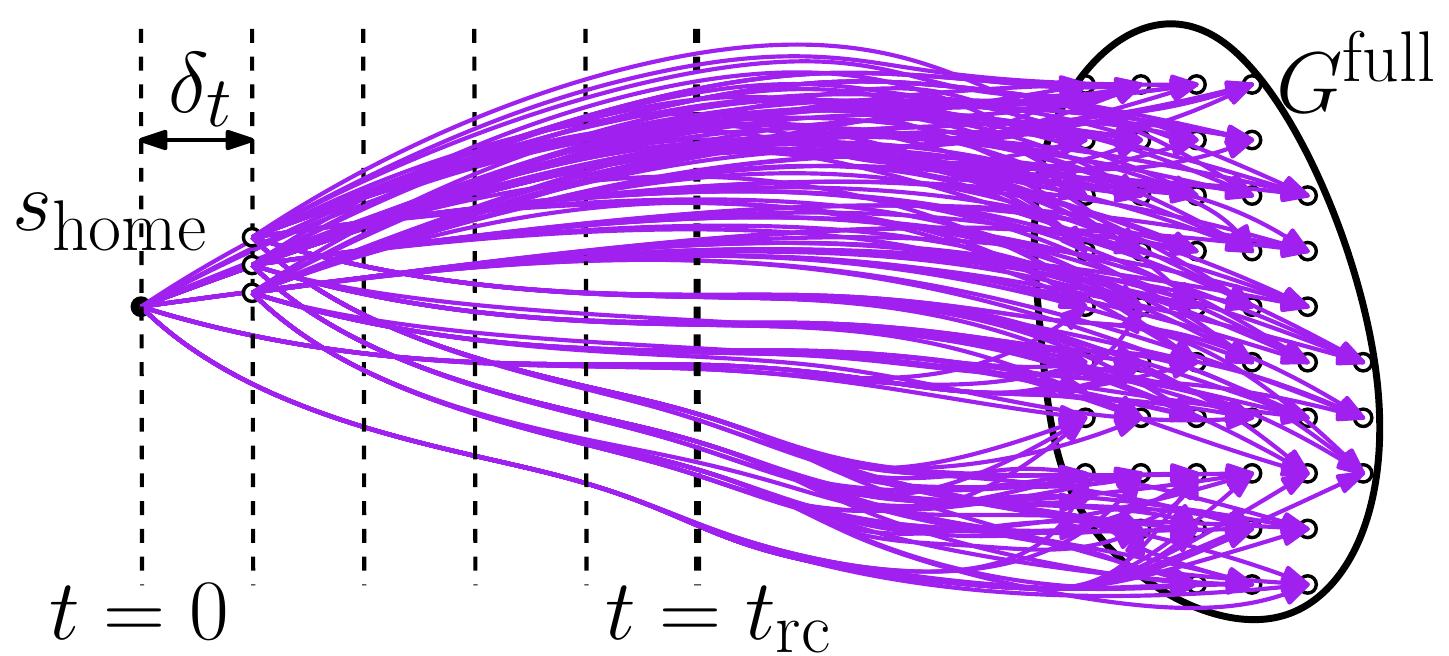}
        \caption{}
        \label{fig:naive2}
    \end{subfigure}
    \caption{
    \CaptionTextSize
    The figures show paths discretized from timesteps $t_0$ to~$\Trc$ with steps of size $\delta_t$.
    (\subref{fig:naive1})~At $t_0$, the algorithm computes~$n_{\rm goal}$ paths, from~\Shome to every $g \in \Gfull$.
    (\subref{fig:naive2})~At $t_1 = \delta_t$, the algorithm computes~$n_{\rm goal}^2$ paths, from all $n_{\rm goal}$ replanable states at $t_1$ to every $g \in \Gfull$ (here we only show paths from three states).
    Thus, the number of paths increases exponentially at every timestep.
    }
    \label{fig:naive}
\end{figure*}

We first compute from \Shome a path $\pi_g$ to every reachable $g \in \Gfull$. 
These paths can be stored in a lookup (hash) table which can be queried in constant time \Tconst (assuming perfect hashing~\cite{czech1997perfect}).
For the straw man approach, since there is no utility of having a \Tbound larger than \Tconst as it only performs look up operations at query time, we have $\Tbound = \Tconst$.
By storing paths to all the goals, every goal is covered by \Shome and this allows us to start executing a path once the perception system gives its initial pose estimate.
However, we need to account for pose update while executing~$\pi_g$. 
%
This only needs to be done up until time~\Trc, since no future improved estimates are expected from the perception system.
Thus, we discretize each path uniformly with resolution $\delta_t$.
%
We call all states that are less than \Trc time from \Shome \emph{replanable states}.

Next, for every replanable state along each path $\pi_g$, we compute a new path to all goals. 
%
This will ensure that all goals are covered by all replanable states. Namely, it will allow to immediately start executing a new path once the goal location is updated by the perception system.
Unfortunately, the perception system may update the goal location more than once. Thus, this process needs to be performed recursively for the new paths as well.

The outcome of the preprocessing stage is a set of precomputed collision-free paths starting at states that are at most $\Trc$ from \Shome and end at goal states.
The paths are stored in a lookup table $\calM: S \times \Gfull \rightarrow \{ \pi_1, \pi_2, \ldots \}$ that can be queried in $\Tbound=\Tconst$ time to find a path from any given $s \in S$ to $g \in \Gfull$.

In the query stage we obtain an estimation $g_1$ of the goal pose by the perception system. 
The algorithm then retrieves the path~$\pi_1(\Shome,g_1)$ (from~\Shome to~$g_1$) from $\calM$ and the robot starts executing~$\pi_1(\Shome,g_1)$.
For every new estimation $g_i$ of the goal pose obtained from the perception system while the system is executing path $\pi_{i-1}(s_{i-1},g_{i-1})$, the algorithm retrieves from $\calM$ the path $\pi_i(s_i,g_i)$ from the nearest state~$s_i$ along $\pi_{i-1}(s_{i-1},g_{i-1})$ that is least $\Tbound$ away from~$s_{i-1}$. The robot then starts executing~$\pi_i(s_i,g_i)$ once it reaches~$s_i$.
Hence, the straw man algorithm is trivially a CTMP-complete algorithm.

Clearly, every possible goal is covered for every possible configuration that the robot might be in during execution by this brute force approach, however it requires a massive amount of memory and prepreprocessing time.
Let $n_{\rm goal} = \vert \Gfull \vert$ be the number of goals and
$\ell$ be the number of states between \Shome and the state that is \Trc time away.
This approach requires precomputing and storing $O(n_{\rm goal}^\ell)$ paths which is clearly infeasible (see Fig.~\ref{fig:naive}).
In the next sections, we show how we can dramatically reduce the memory footprint of the approach without compromising on the system's capabilities.

\subsection{Algorithmic approach}
While the straw man algorithm presented  allows for planning to any goal pose $ g \in \Gfull$ within \Tconst time, its memory footprint is prohibitively large.
We suggest to reduce the memory footprint by building on the observation that many paths to close-by goals traverse very similar parts of the configurations space. Instead of generating a plan strictly within $\Tconst$ time, our approach trades off preprocessing efforts with the bound on the planning time and guarantees to generate a solution within the user-specified time bound $\Tbound (> \Tconst)$.


The key idea of our approach is that instead of computing (and storing) paths to all reachable goals in \Gfull, we compute a relatively small subset of so-called ``root paths" that can be reused in such a way that we can still cover \Gfull fully. Namely, at query time, we can reuse these paths to plan to any $g\in \Gfull$ within \Tbound. The idea is illustrated in Fig.~\ref{fig:crp}.

First, we compute a set of root paths $\{\Pi_1, \ldots, \Pi_k \}$ from~\Shome to cover~\Gfull by \Shome (here we will have that $k \ll n_{\rm goal})$ 
Next, for all replanabale states along these root paths, the algorithm recursively computes additional root paths so that their reachable goals are also covered.
During this process, additional root paths are computed only when the already existing set of root paths does not provide enough guidance to the search to cover \Gfull i.e., to be able to compute a path to any $g \in \Gfull$ within \Tbound.
%
%
The remainder of this section formalizes these ideas.

\subsection{Algorithmic building blocks}
\label{subsec:alg_building_blocks}
We start by introducing the algorithmic building blocks that we use.
Specifically, we start by describing the motion planner \calP  that is used to compute the root paths 
and then continue to describe how they can be used as \emph{experiences} to efficiently compute paths to other goals.
We implemented two types of motion planners. One operates in a time-configuration space and the other is a kinodynamic motion planner that plans with additional state dimensions of joint velocities in order to satisfy the kinodynamic constraints of the robot. While the former planning framework is simpler and works successfully on a physical robot in our experiments, the latter might be more desirable for operating a robot closer to its maximum performance limits.

We use a heuristic search-based planning approach with motion primitives (see, e.g,~\cite{CCL10,CSCL11,LF09})
as it allows for deterministic planning time which is key in our domain.
Moreover, such planners can easily handle under-defined goals as we have in our setting---we define a goal as a grasp pose for the goal object. The grasp pose for a target object $o$ is manually selected.
\subsubsection{Time-Configuration Motion Planner}
\hfill\\
\textbf{State space and graph construction.}
We define a state~$s$ as a pair $(q,t)$ where $q = (\theta_1, ..., \theta_n)$ is a configuration represented by the joint angles for an $n$-DOF robot arm (in our setting $n=$7) and $t$ is the time associated with $q$.
Given a state~$s$ we define two types of motion primitives which are short kinodynamically-feasible motions that the robot can execute.
%

The first type of motion primitives are predefined primitives. These are small individual joint movements in either direction as well as \emph{wait} actions.
These primitives have non-uniform resolution. For each joint, we define two  motion primitives of distance $\pm$4$^{\circ}$. In addition, for the first four of the seven robot joints we define  two additional  primitives each, of distance $\pm$7$^{\circ}$. We only allow moving one joint at a time which makes it a total of 23 predefined primitives.
For each motion primitive, we compute its duration by using a nominal constant velocity profile for the  joint that is moved.
%

The second type of primitives are dynamic primitives. They are generated by the search only at the states that represent the arm configurations where the end effector is close to the object. These primitives correspond to the actual grasping of the object while it is moving.
The dynamic primitives are generated by using a Jacobian pseudo inverse-based control law similar to what~\cite{menon2014motion} used. 
The desired velocity of the end effector is computed for which the end-effector minimizes the distance to the grasp pose. Once the gripper encloses the object, it moves along with the object until the gripper is closed. Some examples of the dynamic primitives are shown in Fig.~\ref{fig:amp}. During the search, all motion primitives are checked for validity with respect to collision and joint limits.

\begin{figure}[t]
    \centering
        \includegraphics[width=0.24\textwidth]{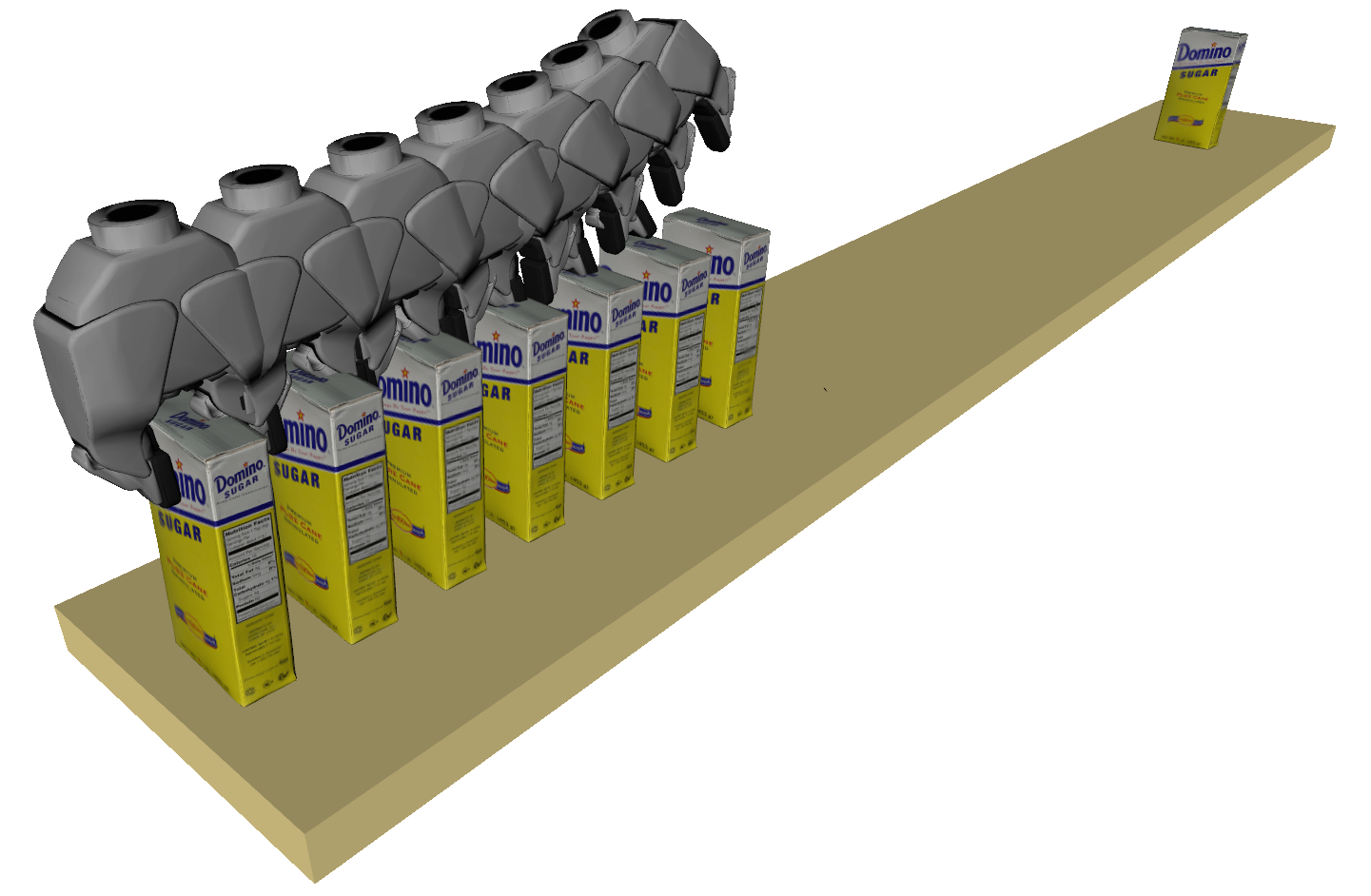}
        \includegraphics[width=0.24\textwidth]{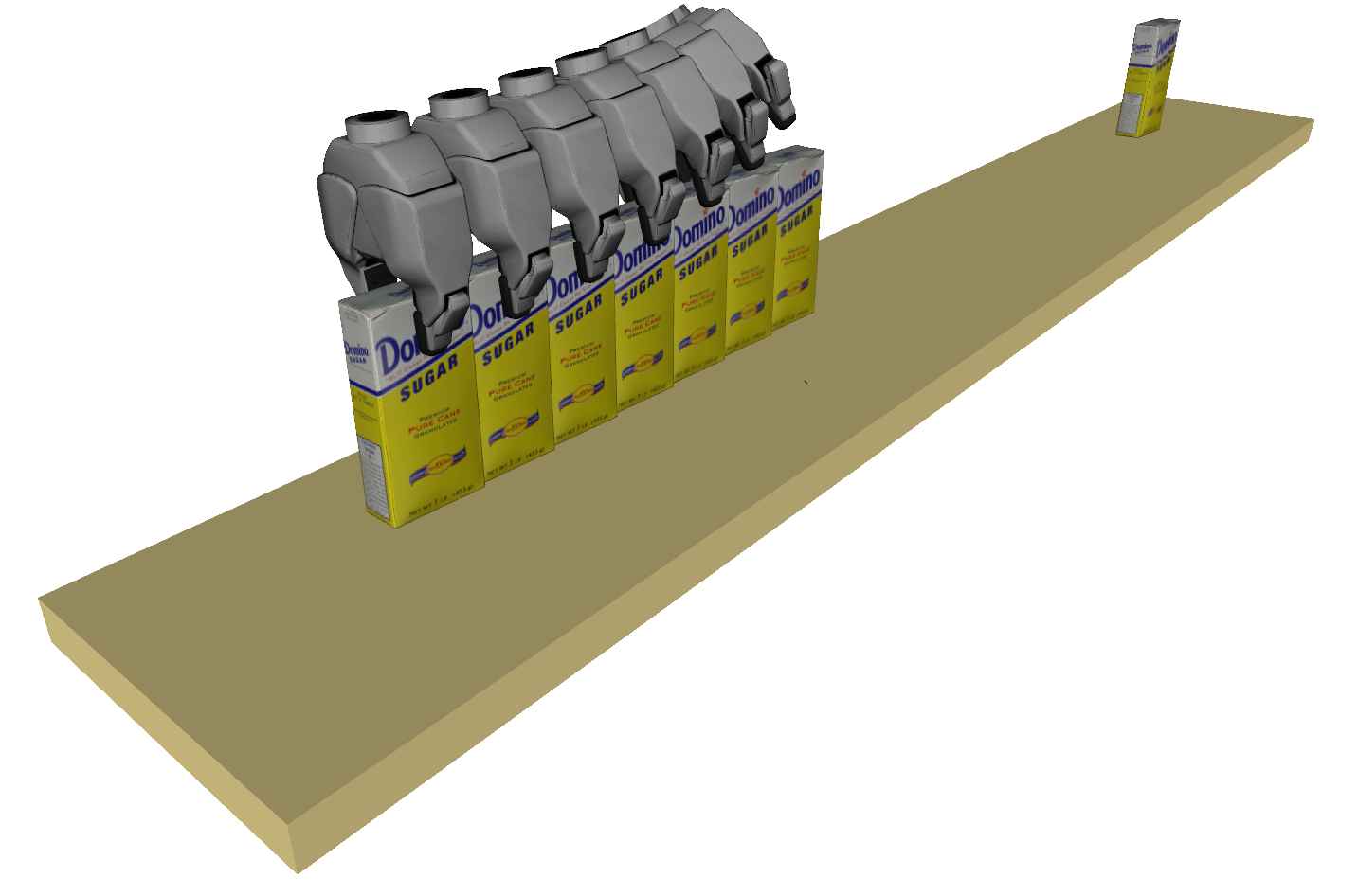}
    \caption{
    The figures show the dynamic motion primitive for two different initial poses of the sugar box. (Only the gripper poses are visualized.) The sugar box shown at the rear of the conveyor belt in the examples depicts its initial pose.
    }
    \label{fig:amp}
\end{figure}

\textbf{Heuristic Function}
The search is guided by an efficient and fast-to-compute heuristic function which in our case has two components.
The first component drives the search to intercept the object at the right time and 
the second component guides the search to correct the orientation of the end effector as it approaches the object. 
Mathematically, our heuristic function is given by
\begin{equation} \label{eq:h1}
 h(s,g) = \max (\Delta t(s,g), \lambda \cdot \textsc{AngleDiff}(s,g)).
\end{equation}
Here, $\Delta t(s,g)$ is the expected time to intercept the object which can be analytically computed from the positions of the target object and the end-effector, the velocity of the target object and the speed of the end-effector. We pick a nominal speed for the end-effector to solve the problem. \textsc{AngleDiff}($s,g$) gives the magnitude of angular difference between the end-effector's current pose and the grasp pose. The coefficient $\lambda$ is used as a weighting factor.

\subsubsection{Kinodynamic Motion Planner}
\hfill\\
The time-configuration motion planner does not guarantee that the plan when transformed into a trajectory will satisfy robot's torque limits. To this end, we extend the previous motion planner framework to be able to plan within the kinodynamic constraints of the robot, while still being able to handle the conveyor speed of 0.2$m/s$. In this section we describe the modified state space and the heuristic function for the kinodynamic motion planner.

\textbf{State space and graph construction.}
We modify the state space to include joint velocities. A state $s$ is a tuple $(q,\dot{q},t)$ where $\dot{q} = (\dot{\theta_1},..., \dot{\theta_n})$ represent velocities for each joint. The dimensionality of the state space hence becomes 15 in our experiments. Similar to $q$, the velocities $\dot{q}$ are also discretized.

We modify the set of predefined motion primitives. The primitives are specified as magnitudes of accelerations for each joint. Specifically, to generate a successor of a state, for the desired accelerations (specified in the primitive), first we use inverse dynamics to compute the required torques. Second, we integrate using the forward dynamics for a fixed time step $dt$ to get the resulting state using Runge-Kutta integration. We use the Orocos Kinematics and Dynamics Library (KDL) to solve for forward/inverse dynamics and the integration\footnote{\url{KDL: https://orocos.org/kdl.html}.}.
Let a motion primitive be specified by a vector of accelerations $\ddot{q}$ of size $n$. The two steps to compute the primitive are

\begin{gather*}
\tau = \textsc{ComputeInverseDynamics}(s,\ddot{q}) \\
s' = \textsc{Integrate}(s, \tau, dt)
\end{gather*}

Since the robot dynamics is a function of the state $s$ of the robot, these primitives need to be computed on-the-fly during search. In addition to performing kinematic and collision checks to verify motion validity, we discard successors for which $\tau$ exceeds the robot's torque or velocity limits.

In our experiments, we use 6 motion primitives for each of the 7 joints. These are accelerations $\pm$(4, 8, 12) $\deg/s^2$. We only accelerate or decelerate one joint at a time under these acceleration profiles, thus resulting in 42 primitives. In addition to these primitives, we use a ``coasting" primitive that assigns zero acceleration to all the joints.

\textbf{Heuristic Function.}
The heuristic function we used for the time-configuration space planner gives no guidance for the velocity profiling which is crucial in the case of kinodynamic planning. We, therefore, modify the heuristic in Eq.~\ref{eq:h1} by introducing an additional term $\Delta \dot{x}$ that guides the search with respect to the velocity profile.

\begin{equation}
\Delta \dot{x}(s) = \|\mathbf{\dot{x^o} - \dot{x^e}}(s)\|.
\end{equation}

Namely, $\Delta \dot{x}$ is the magnitude of the difference of the target object's velocity $\mathbf{\dot{x^o}}$ and the robot end-effector's velocity $\mathbf{\dot{x^e}}$ at state $s$ in 3D. $\mathbf{\dot{x^e}}$ is computed using forward velocity kinematics.

The new heuristic function is given by

\begin{multline}
 h(s,g) = \max (\Delta t(s,g), \\
 \lambda_1 \cdot \textsc{AngleDiff}(s,g) + \lambda_2 \cdot \Delta \dot{x}(s)).
\end{multline}

where $\lambda_1$ and $\lambda_2$ are the weighting factor. Intuitively, this additional term guides the search to match the end-effector velocity with $o$'s velocity as the end-effector approaches the object. This increases the likelihood of generating a dynamic primitive that satisfies the kinodynamic constraints of the robot. 

\subsubsection{Graph Search}
The states and the transitions implicitly define a graph $\calG = (S,E)$ where $S$ is the set of all states and~$E$ is the set of all transitions defined by the motion primitives. We use Weighted A* (wA*)~\cite{pohl1970heuristic} to find a path in $\calG$ from a given state~$s$ to a goal $g$. 
wA* is a suboptimal heursitic search algorithm that allows a tradeoff between optimality and greediness by inflating the heuristic function $h$ by a given weight~$w$. The cost of an edge is the time of its traversal.

\subsubsection{Planning with Experience Reuse}
We now show how previously-computed paths which we named as root paths can be reused as experiences in our framework. Given a heuristic function $h$, we define for a root path $\Pi$ and a goal $g \in \Gfull$ the \emph{shortcut} state $\Ssc (\Pi,g)$ as the state on the path $\Pi$ that is closest to~$g$ with respect~$h$.
Namely,
\begin{equation}
\Ssc (\Pi,g) := \argmin\limits_{s_i \in \Pi} h(s_i, g).
\end{equation}
Now, when searching for a path to a goal $g \in \Gfull$ using root path $\Pi$ as an experience, we add $\Ssc (\Pi,g)$ as a successor for any state along~$\Pi$
(subject to the constraint that the path along~$\Pi$ to \Ssc is collision free). In this manner the search reuses previous experience to quickly reach a state close to $g$.

\subsection{Algorithmic details}
We are finally ready to describe our algorithm describing first the preprocessing phase and then  the query phase.

\subsubsection{Preprocessing}
\begin{figure*}[t]
    \centering
    \begin{subfigure}{.25\textwidth}
        \includegraphics[width=\textwidth]{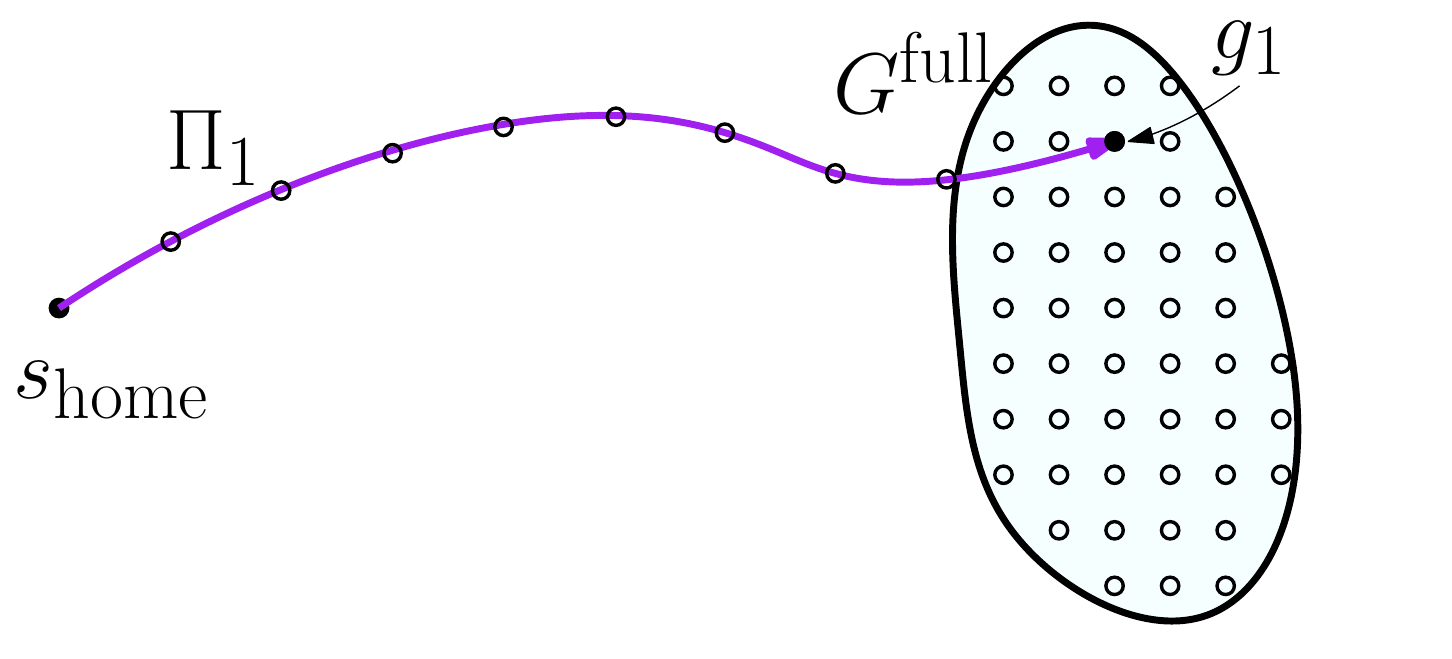}
        \caption{}
        \label{fig:crp1}
    \end{subfigure}
    \hspace{4mm}
    \begin{subfigure}{0.25\textwidth}
        \includegraphics[width=\textwidth]{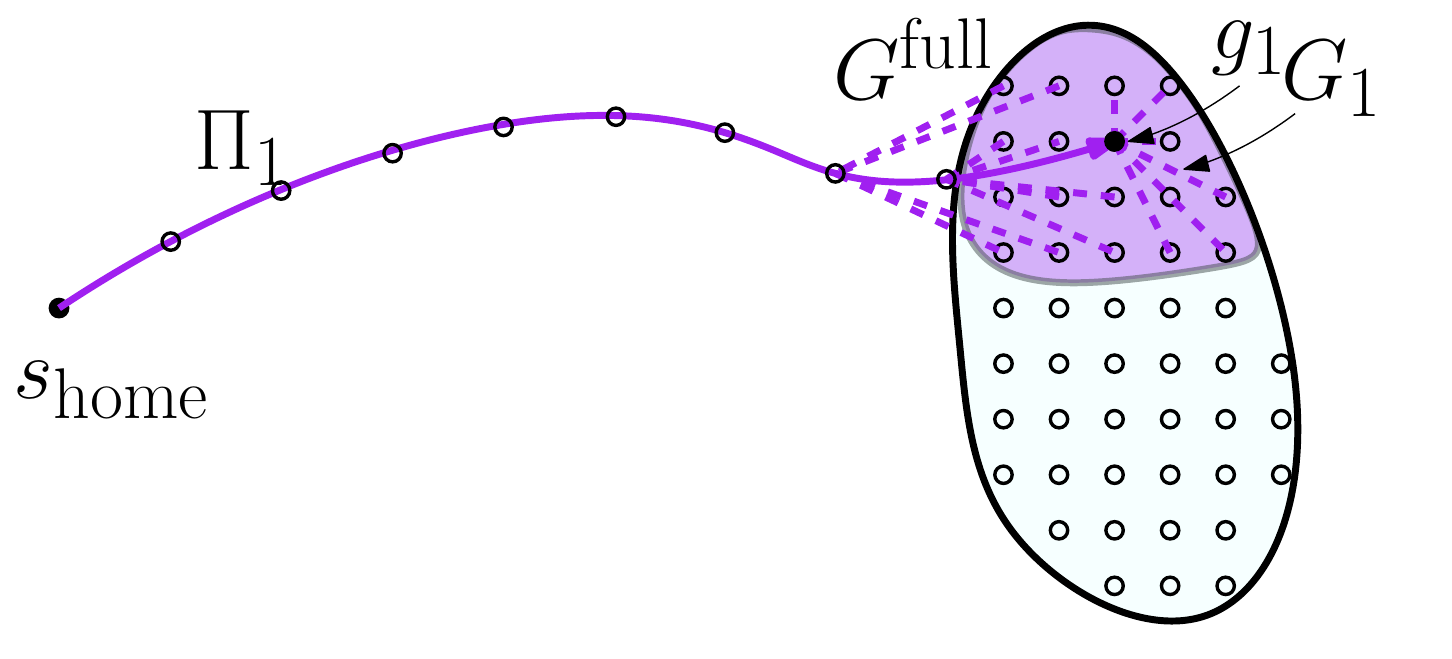}
        \caption{}
        \label{fig:crp2}
    \end{subfigure}
    \hspace{4mm}
    \begin{subfigure}{0.25\textwidth}
        \includegraphics[width=\textwidth]{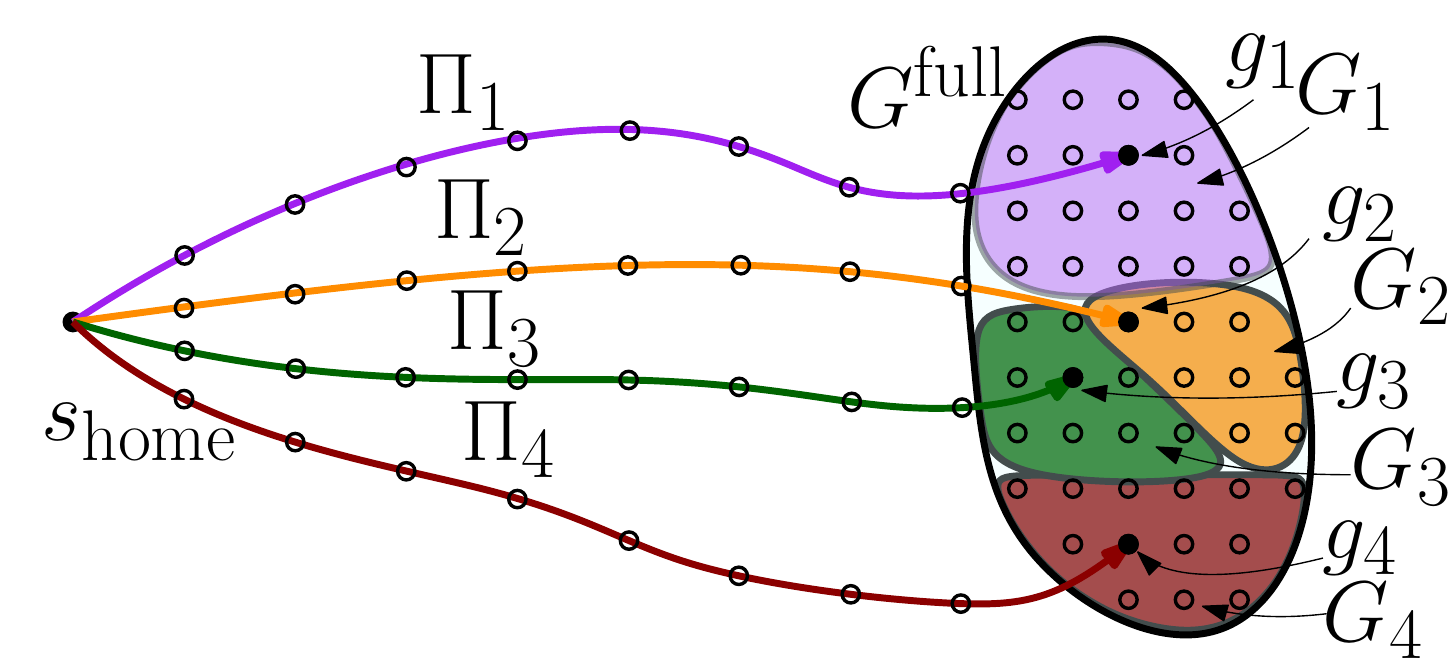}
        \caption{}
        \label{fig:crp3}
    \end{subfigure}
    \caption{\CaptionTextSize
    First step of the preprocessing stage.
    (\subref{fig:crp1})~A goal $g_1$ is sampled and the root path $\Pi_1$ is computed between \Shome and $g_1$.
    (\subref{fig:crp2})~The set $G_1 \subset \Gfull$ of all states that can use $\Pi_1$ as an experience is computed and associated with $\Pi_1$.
    (\subref{fig:crp3})~The goal region covered by four root paths from \Shome after the first step of the preprocessing stage terminates.
    }
    \label{fig:crp}
\end{figure*}

Our preprocessing stage starts by sampling a goal~$g_1 \in \Gfull$ and computing a root path~$\Pi_1$ from~$\Shome$ to~$g_1$. We then associate with~$\Pi_1$ the set of goals ~$G_1 \subset \Gfull$ such that~$\Pi_1$ can be used as an experience in reaching any $g_j \in G_1$ within~\Tbound.\footnote{To account for the lower bound~$\Tconst$ time that is required for the query phase which is consumed in operations, such as hash table lookups etc., the time $\Tconst$ is subtracted from $\Tbound$ for the experience-based planner, to ensure that the overall query time is bounded by $\Tbound$. We use a conservative estimate of $\Tconst$ in our experiments.}
Additionally, the experience-based planner is constrained to reuse the root path upto \Trc.
Thus, all goals in~$G_1$ are covered by~\Shome.
We then repeat this process but instead of sampling  a goal from \Gfull, we sample from $\Gfull \setminus G_1$, thereby removing covered goals from \Gfull in every iteration.
At the end of this step, we obtain a set of root paths. 
Each root path~$\Pi_i$ is associated with a goal set $G_i \subseteq \Gfull$ such that 
(i)~$\Pi_i$ can be used as an experience for planning to any $g_j \in G_i$ in~\Tbound and 
(ii)~$\bigcup_i G_i = \textsc{Reachable}(\Shome, \Gfull)$ (i.e., all reachable goals for \Shome in \Gfull).
Alg.~\ref{alg:step1} details this step (when called with arguments ($\Shome,\Gfull$)). It also returns a set of unreachable goals that are left uncovered.
The process is illustrated in Fig.~\ref{fig:crp}.

\begin{algorithm}[t]
\caption{Plan Root Paths}
\label{alg:step1}
    \AlgFontSize
\begin{algorithmic}[1]
\Procedure{PlanRootPaths}{$s_{\textrm{start}}, \Guncov$}
\State $\Psi_{\Sstart} \leftarrow \emptyset$   \Comment{a list of pairs ($\Pi_i, G_i)$}
\State $\Guncov_{\Sstart} \leftarrow \emptyset$; \hspace{3mm}
       $i = 0$
\While{$\Guncov \neq \emptyset$}
        \Comment{until all reachable goals are covered}
    \State $g_i \leftarrow$\textsc{SampleGoal}($\Guncov$)
    \State $\Guncov \leftarrow \Guncov \setminus \{g_i\}$
    
    \If {$\Pi_i \leftarrow$ \textsc{PlanRootPath}($s_{\textrm{start}}, g_i$)} \Comment{planner succeeded}
        \State $G_i \leftarrow \{ g_i \}$   \Comment{goals reachable}
        \For {\textbf{each} $g_j \in \Guncov$}
            \If {$\pi_j \leftarrow$\textsc{PlanPathWithExperience}($s_{\textrm{start}},g_j,\Pi_i$)}
                \State $G_i \leftarrow G_i \cup \{g_j\}$
                \State $\Guncov \leftarrow \Guncov \setminus \{g_j\}$
            \EndIf
        \EndFor
        \State $\Psi_{\Sstart} \leftarrow \Psi_{\Sstart} \cup \{ (\Pi_i, G_i)\}$; \hspace{3mm}
        $i \leftarrow i + 1$
        
    \Else
        \State $\Guncov_{\Sstart} \leftarrow \Guncov_{\Sstart} \cup \{g_i\}$ \Comment{goals unreachable}
    \EndIf
\EndWhile
\State \textbf{return} $\Psi_{\Sstart}, \Guncov_{\Sstart}$
\EndProcedure
\end{algorithmic}
\end{algorithm}

%
\begin{algorithm}[t]
\caption{Preprocess}\label{alg:preprocess}
    \AlgFontSize
\begin{algorithmic}[1]
\Procedure{TryLatching}{$s,\Psi_{\Shome}\Guncov,\Gcov$}
        \For {\textbf{each} $(\Pi_i, G_i) \in \Psi_{\Shome}$}
        \label{alg:preprocess:line:latch1a}
            \If{\textsc{CanLatch}($s,\Pi_i$)}
                \State $\Guncov \leftarrow \Guncov \setminus G_i$
                \State $\Gcov \leftarrow \Gcov \cup G_i$
                \label{alg:preprocess:line:latch1b}
            \EndIf
        \EndFor
\State \textbf{return} $\Guncov, \Gcov$
\vspace{2mm}
\EndProcedure
\Procedure{Preprocess}{$\Sstart,\Guncov,\Gcov$}
\State $\Psi_{\Sstart}, G^{\textrm{uncov}}_{\Sstart} \leftarrow$ \textsc{PlanRootPaths}($\Sstart,\Guncov$)
{\color{blue}{ \State \textbf{if} $\Sstart = \Shome$ \textbf{then} $\Psi_{\Shome} = \Psi_{\Sstart}$}}


\State $G^{\textrm{cov}}_{\Sstart} \leftarrow 
    \Gcov \cup (\Guncov \setminus G^{\textrm{uncov}}_{\Sstart})$
\If{$t(s_{\textrm{start}}) \leq \Trc$}

\For {\textbf{each} $(\Pi_i, G_i) \in \Psi_{\Sstart}$} \label{loop1}
    \State $G_i^{\textrm{cov}} \leftarrow G_i$;
            \hspace{2mm}
           $G_i^{\textrm{uncov}} \leftarrow G^{\textrm{cov}}_{\Sstart} \setminus G_i$; \label{uncov_init}
            \hspace{2mm}

    \For {\textbf{each} $s \in \Pi_i$ ({from last to first})} \Comment{states up to $\Trc$} \label{loop2}
\textcolor{blue}{
        \State $\Guncov_i,\Gcov_i \leftarrow$ \textsc{TryLatching}($s,\Psi_{\Shome},\Guncov_i,\Gcov_i$)
        \If{$G_i^{\textrm{uncov}} = \emptyset$}
            \State \textbf{break}
        \EndIf
} 
        \State $G_i^{\textrm{uncov}},G_i^{\textrm{cov}} \leftarrow$ \textsc{Preprocess}($s,G_i^{\textrm{uncov}},G_i^{\textrm{cov}}$)    \label{recursion}
        \If{$G_i^{\textrm{uncov}} = \emptyset$} \label{terminate}
            \State \textbf{break}
        \EndIf
 
    \EndFor
\EndFor
\EndIf
\State \textbf{return} $G^{\textrm{uncov}}_{\Sstart}, G^{\textrm{cov}}_{\Sstart}$
\EndProcedure
\end{algorithmic}
\end{algorithm}

%


So far we explained the algorithm for one-time planning when the robot is at \Shome ($t = 0$); we now need to allow for efficient replanning for any state $s$ between $t = 0$ to $\Trc$. In order to do so, we iterate through all the states on these root paths and add additional root paths so that these states also cover their respective reachable goals. This has to be done recursively since newly-added paths generate new states which the robot may have to replan from. The complete process is detailed in Alg.~\ref{alg:preprocess}.
The \textsc{Preprocess} procedure takes in a state \Sstart, the goal region that it has to cover \Guncov and region that it already has covered \Gcov. Initially \textsc{Preprocess} is called with arguments ($\Shome, \Gfull, \emptyset$) and it runs recursively until no state is left with uncovered reachable goals.
%

At a high level, the algorithm iterates through each root path~$\Pi_i$ (loop at line~\ref{loop1}) and for each state $s \in \Pi_i$ (loop at line~\ref{loop2}) the algorithm calls itself recursively (line~\ref{recursion}). The algorithm terminates when all states cover their reachable goals. The pseudocode in blue constitute an additional optimization step which we call ``latching" and is explained later in Sec.~\ref{latching}.

In order to minimise the required computation, the algorithm leverages two key observations:

\begin{enumerate}[label={\textbf{O\arabic*}},leftmargin=0.75cm]
    \item \label{obs:1}
    If a goal is not reachable from a state $s \in \Pi$, it is not reachable from all the states after it on $\Pi$.
    \item \label{obs:2}
    If a goal is covered by a state $s \in \Pi$, it is also covered by all states preceeding it on $\Pi$.
\end{enumerate}

\ref{obs:1} is an assumption that we make about the planner \calP. We use \ref{obs:1} to initialize the uncovered set for any state; instead of attempting to cover the entire \Gfull for each replanable state~$s$, the algorithm only attempts to cover the goals that could be reachable from $s$, thereby saving computation.
%
\ref{obs:2} is used by iterating backwards on each root path (loop at line~\ref{loop2}) and for each state on the root path only considering the goals that are left uncovered by the states that appear on the path after it.

Specifically,~\ref{obs:2} is used to have a single set of uncovered goals~$\Guncov_i$ for all states that appear on $\Pi_i$ instead of having individual sets for each state and the goals that each $s \in \Pi_i$ covers in every iteration of loop~\ref{loop2} are removed from ~$\Guncov_i$.
\ref{obs:1} is used to initialize $\Guncov_i$ (in line~\ref{uncov_init}). Namely, it is initialized not by the entire \Gfull but by the set of goals covered by \Sstart. $G_i$ is excluded since it is already covered via $\Pi_i$. The iteration completes either when all goals in $\Guncov_i$ are covered (line~\ref{terminate}) or the loop backtracks to \Sstart.
The process is illustrated in Fig.~\ref{fig:pl_no_latching}

Thus, as the outcome of the preprocessing stage, a map $\calM: S \times \Gfull \rightarrow \{ \Pi_1, \Pi_2, \ldots \}$ is constructed that can be looked up to find which root path can be used as an experience to plan to a goal $g$ from a state $s$ within \Tbound.

\begin{figure*}[t]
    \centering
    \begin{subfigure}{.245\textwidth}
        \includegraphics[width=\textwidth]{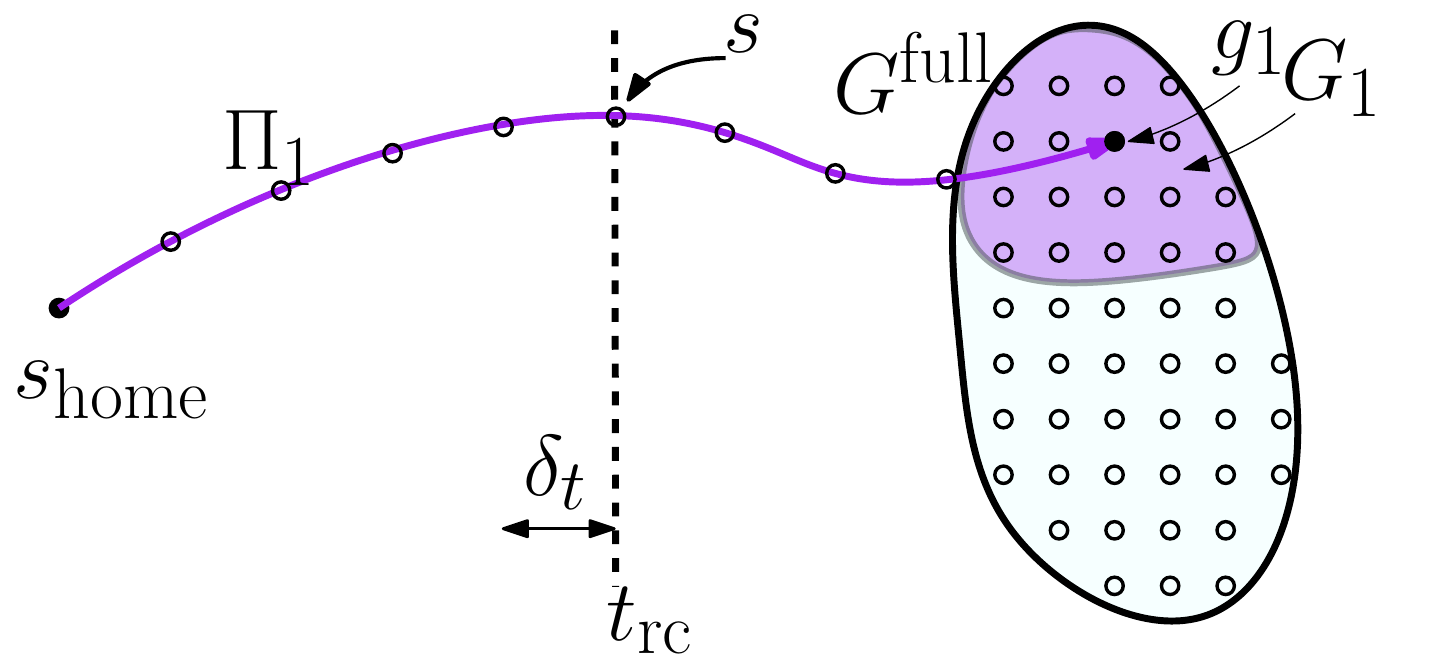}
        \caption{}
        \label{fig:pre1}
    \end{subfigure}
    \begin{subfigure}{.245\textwidth}
        \includegraphics[width=\textwidth]{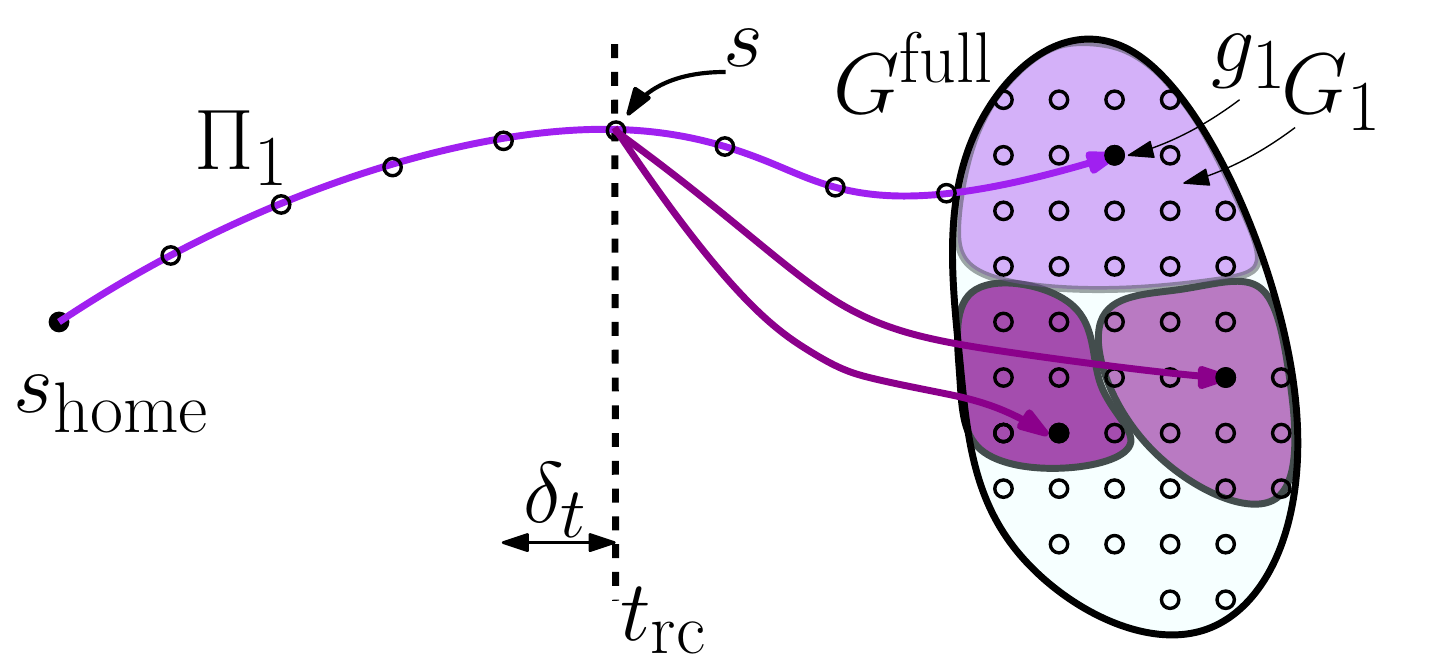}
        \caption{}
        \label{fig:pre2}
    \end{subfigure} 
    \begin{subfigure}{.245\textwidth}
        \includegraphics[width=\textwidth]{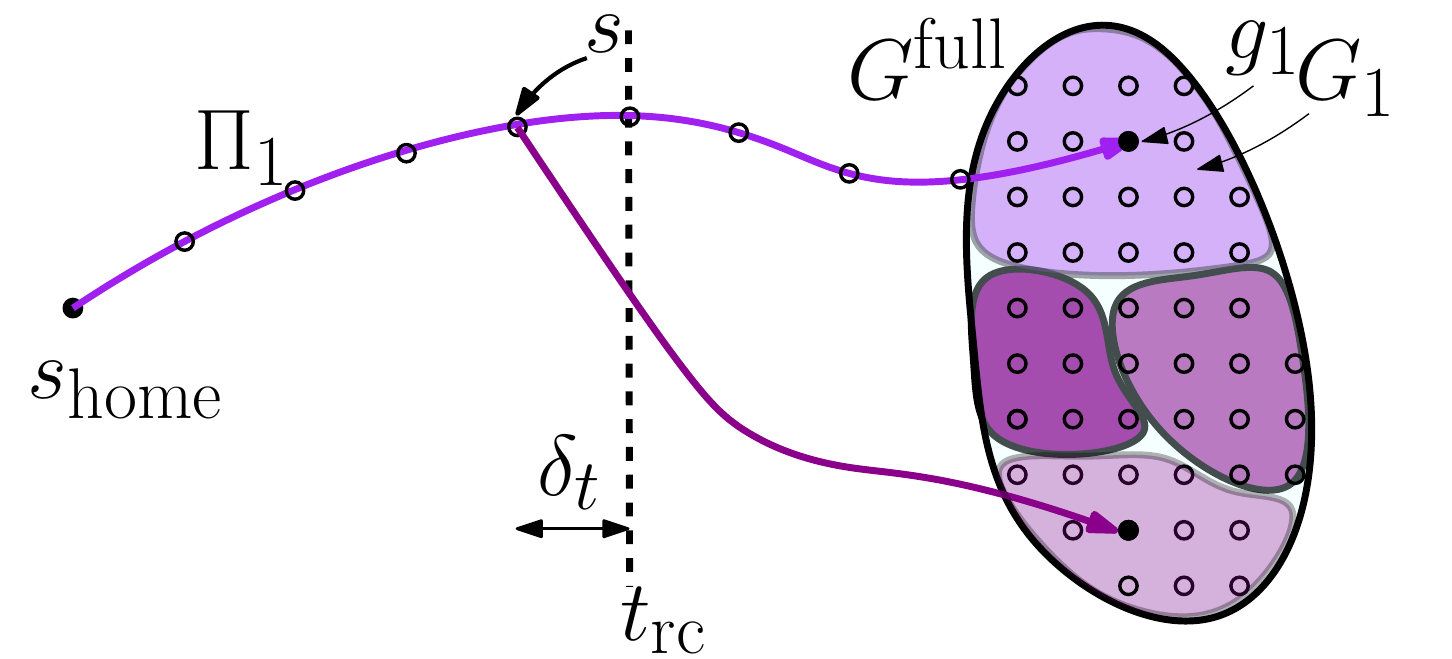}
        \caption{}
        \label{fig:pre3}
    \end{subfigure}
    %
    \begin{subfigure}{.245\textwidth}
        \includegraphics[width=\textwidth]{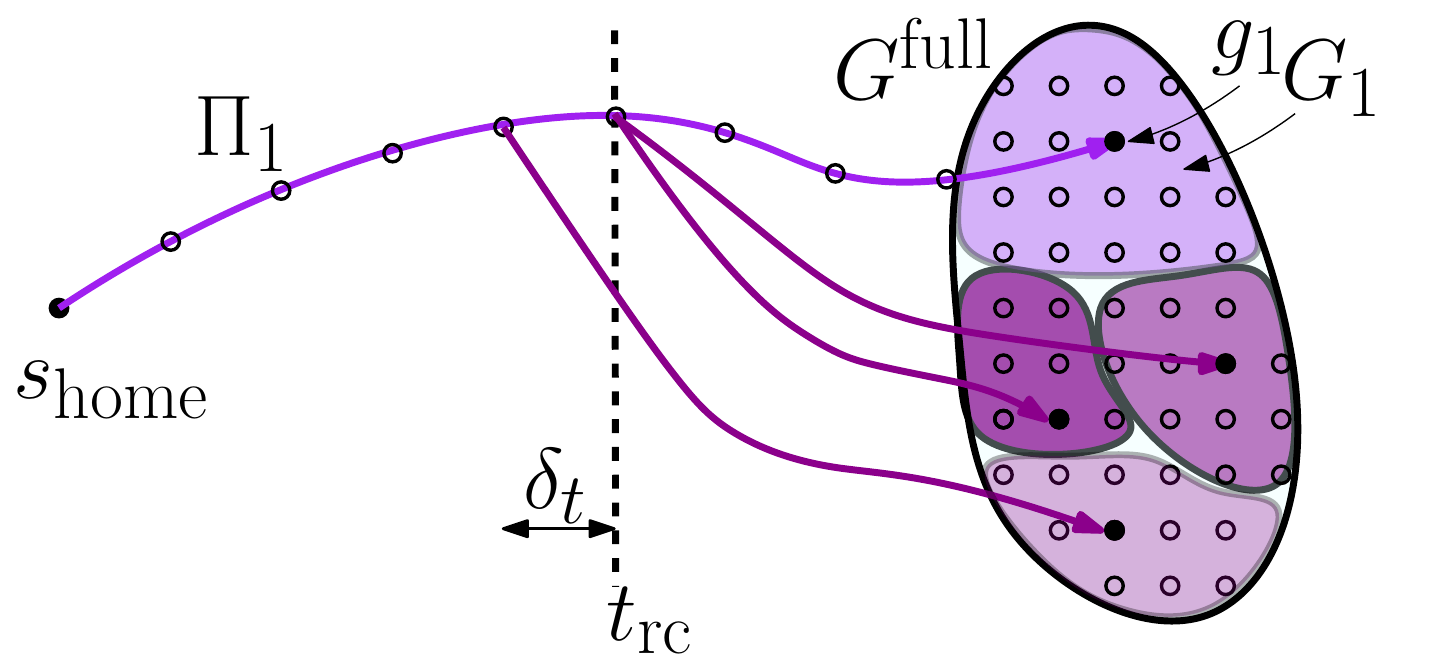}
        \caption{}
        \label{fig:pre4}
    \end{subfigure}
    \caption{\CaptionTextSize
    Preprocess loop for $\Pi_1$ without latching.
    (\subref{fig:pre1})~Initially the state $s$ covers $G_1$ via $\Pi_1$. 
    (\subref{fig:pre2})~New root paths are computed from $s$ to cover remaining uncovered region.
    (\subref{fig:pre3})~This process is repeated by backtracking along the root path.
    (\subref{fig:pre4})~Outcome of a preprocessing step for one path: \Gfull is covered either by using $\Pi_1$ as an experience or by 
    using newly-computed root paths. 
    }
    \label{fig:pl_no_latching}
\end{figure*}

\begin{figure*}[t]
    \centering
    \begin{subfigure}{.245\textwidth}
        \includegraphics[width=\textwidth]{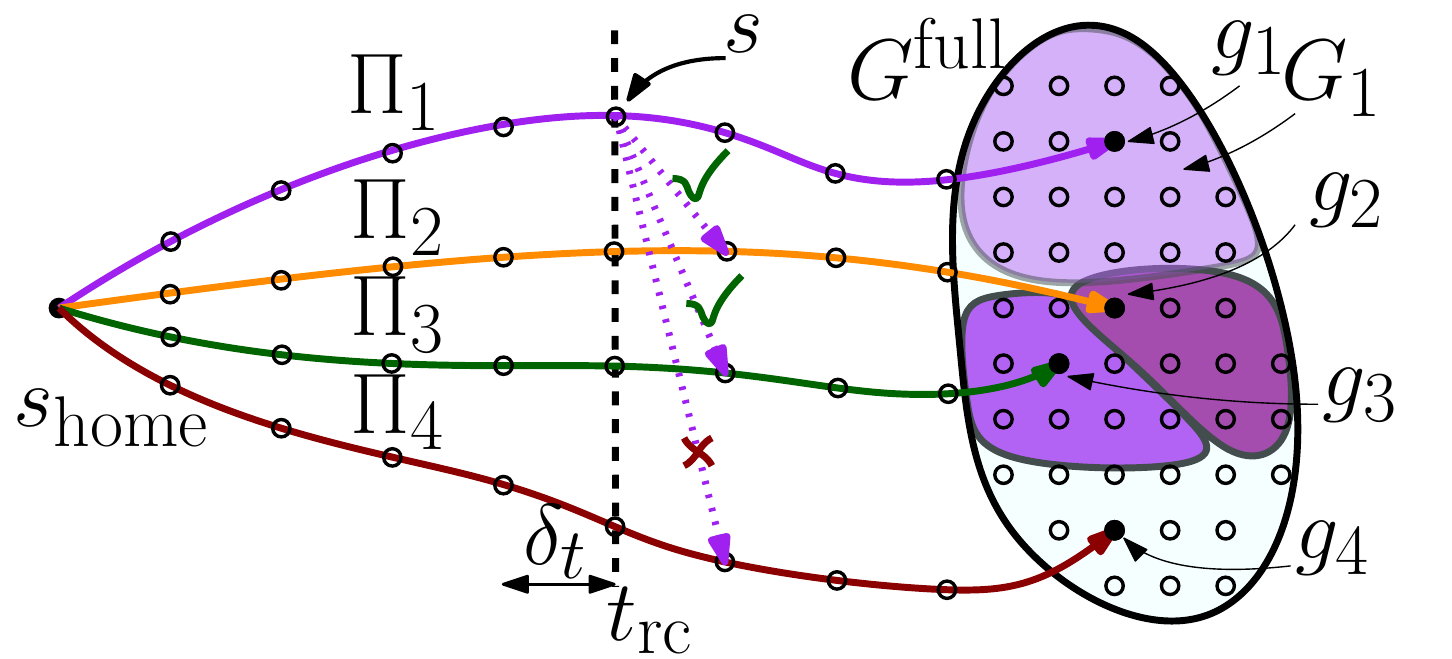}
        \caption{}
        \label{fig:pl1}
    \end{subfigure}
    \begin{subfigure}{.245\textwidth}
        \includegraphics[width=\textwidth]{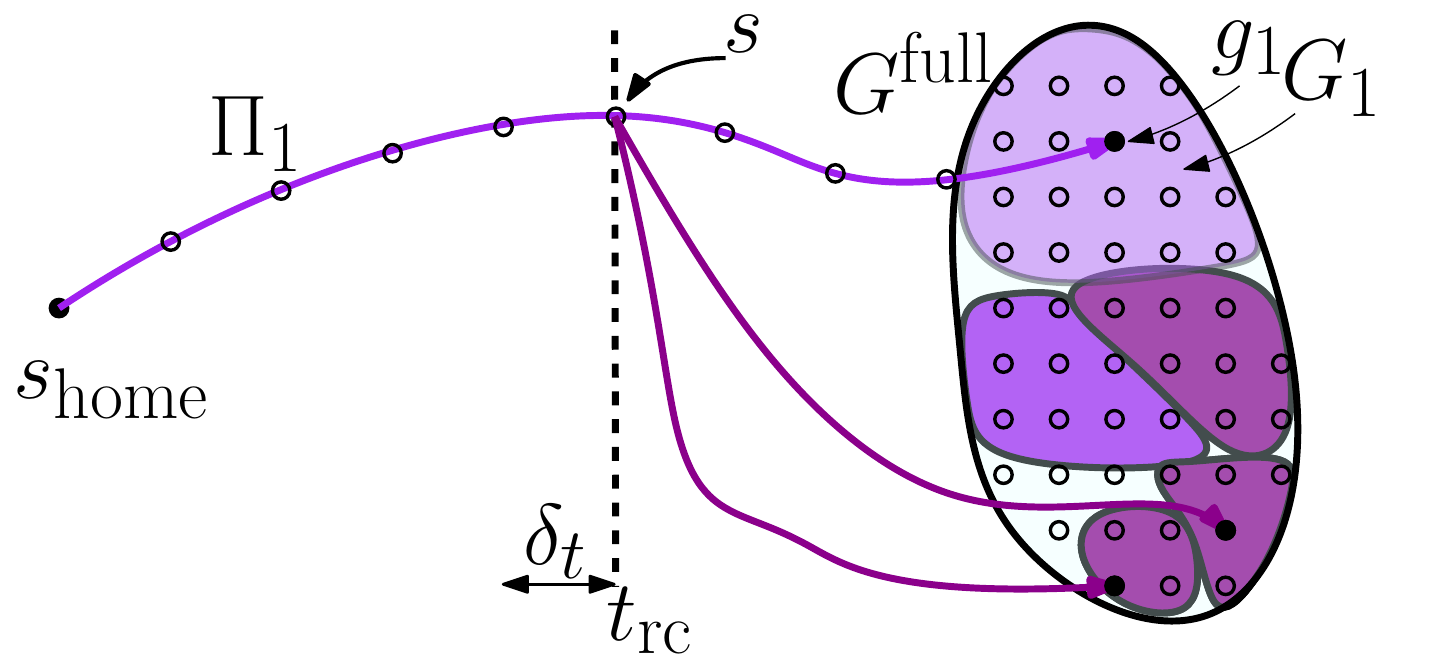}
        \caption{}
        \label{fig:pl2}
    \end{subfigure} 
    \begin{subfigure}{.245\textwidth}
        \includegraphics[width=\textwidth]{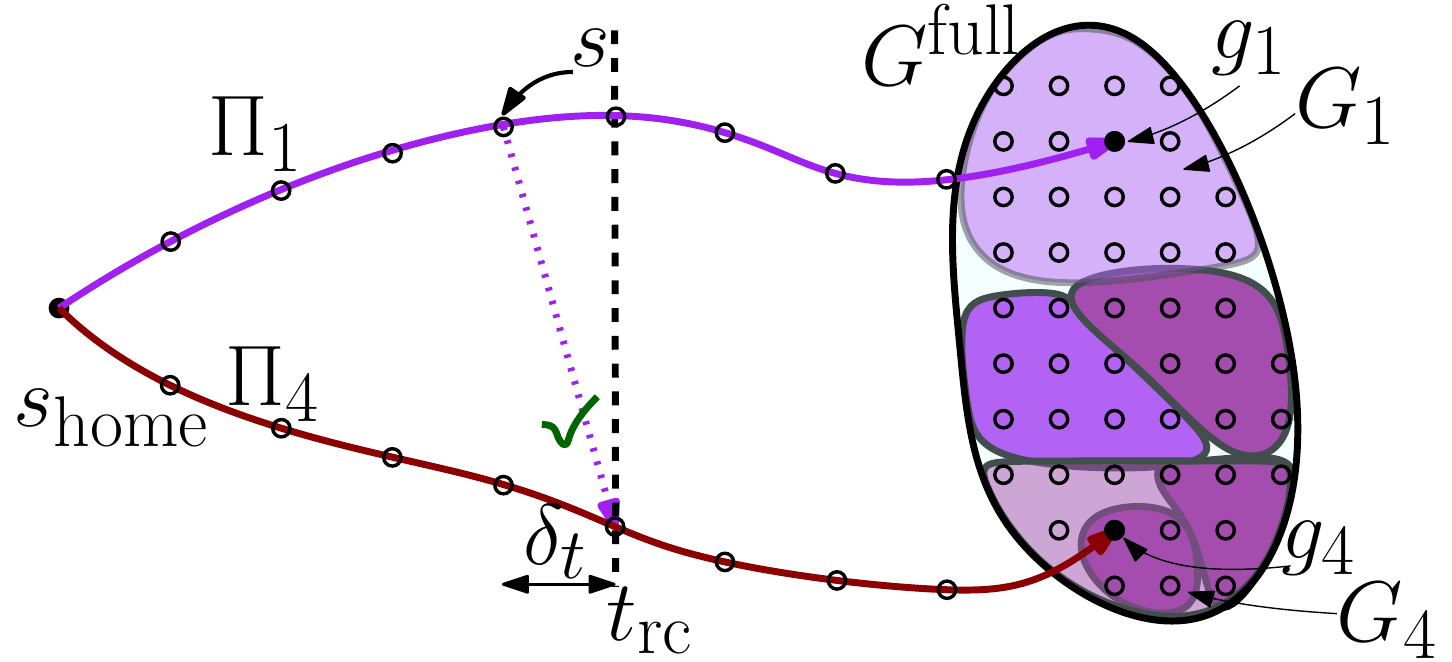}
        \caption{}
        \label{fig:pl3}
    \end{subfigure}
    %
    \begin{subfigure}{.245\textwidth}
        \includegraphics[width=\textwidth]{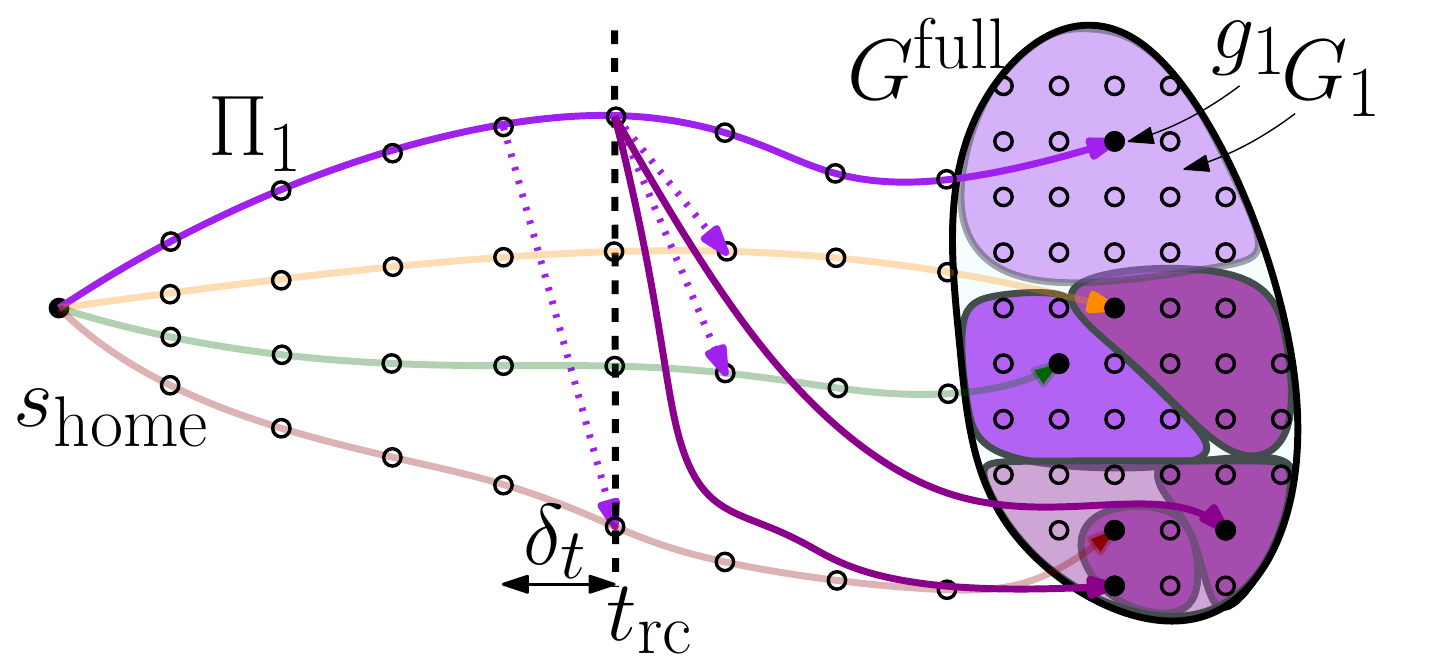}
        \caption{}
        \label{fig:pl4}
    \end{subfigure}
    \caption{\CaptionTextSize
    Preprocess loop for $\Pi_1$ with latching.
    (\subref{fig:pl1})~The algorithm starts by trying to latch on to every other root path; for successful latches, the corresponding goals are removed from uncovered region.
    (\subref{fig:pl2})~New root paths are computed from $s$ to cover remaining uncovered region.
    (\subref{fig:pl3})~This process is repeated by backtracking along the root path.
    (\subref{fig:pl4})~Outcome of a preprocessing step: \Gfull is covered either by using $\Pi_1$ as an experience, 
    latching on to $\Pi_2,\Pi_3$ or  $\Pi_4$ (at different time steps)
    or by 
    using newly-computed root paths. 
    }
    \label{fig:pl_latching}
\end{figure*}

\subsubsection{Query}
Alg.~\ref{alg:query} describes the query phase of our algorithm. Again, the lines in blue correspond to the blue pseudocode in Alg.~\ref{alg:preprocess} for the additional optimization step which is explained in Sec.~\ref{latching}.
Assume that the robot was at a state~$s_{\textrm{curr}}$ while executing a path~$\pi_{\rm curr}$ when it receives a pose update~$g$ from the perception system. Alg.~\ref{alg:query} will be called for the first state \Sstart that is \Tbound ahead of~$s_{\textrm{curr}}$ along $\pi_{\rm curr}$, allowing the algorithm to return a plan before the robot reaches \Sstart.

Alg.~\ref{alg:preprocess} assures that either the first state on ~$\pi_{\rm curr}$ covers $g$ or there exists one state on~$\pi_{\rm curr}$ between~\Sstart and the state at \Trc that covers $g$. The algorithm first checks if the first state on ~$\pi_{\rm curr}$ covers $g$ or not. If it covers $g$ then the corresponding root path is used to find the new path ~$\pi_{\rm new}$. Otherwise, it iterates over each~$s \in \pi_{\rm curr}$ backwards (similar to Alg.~\ref{alg:preprocess}) between~\Sstart and the state at \Trc and finds the one that covers $g$ by quering~$\calM$. Once found, the corresponding root path~$\Pi_{\rm next}$ is used as an experience to plan the path $\pi_{\rm next}$ from $s$ to $g$. Finally the paths $\pi_{\rm curr}$ and $\pi_{\rm next}$ are merged together with $s$ being the transitioning state to return the final path $\pi$.

\begin{algorithm}[t]
\caption{Query}\label{alg:query}  
    \AlgFontSize
\hspace*{\algorithmicindent} \textbf{Inputs:} $\calM, \Shome$
\begin{algorithmic}[1]
\Procedure{PlanPathByLatching}{$s,g,\pi_{\textrm{curr}}$}
\If{$\Pi_{\textrm{home}} \leftarrow \calM (\Shome,g)$ exists} \Comment{lookup root path}
        \If{\textsc{CanLatch}($s,\Pi_{\textrm{home}}$)}
            \State $\pi_{\textrm{home}} \leftarrow$\textsc{PlanPathWithExperience}($s_{\textrm{home}},g,\Pi_{\textrm{home}}$)
            \State $\pi \leftarrow$ \textsc{MergePathsByLatching}($\pi_{\textrm{curr}},\pi_{\textrm{home}}, s$)
            \State \textbf{return} $\pi$
            \label{alg:query:line:c3b}
        \EndIf
    \EndIf
\State \textbf{return failure}
\EndProcedure
\vspace{2mm}
\Procedure{Query}{$g, \pi_{\textrm{curr}},s_{\textrm{start}}$}
    \State $s \leftarrow \pi_{\textrm{curr}}[0]$    \Comment{first state on $\pi_{\textrm{curr}}$}
    \If{$\Pi_{\textrm{curr}} \leftarrow$ $\calM(s,g)$ exists} \Comment{lookup root path}
        \State $\pi_{\textrm{new}} \leftarrow$\textsc{PlanPathWithExperience}($s,g,\Pi_{\textrm{curr}}$)
        \State \textbf{return} $\pi_{\textrm{new}}$
    \EndIf
    \For {\textbf{each} $s \in \pi_{\textrm{curr}}$ (from last to $\Sstart$)} \Comment{states up to $\Trc$}
        \label{alg:query:line:c2a}
         
    \If{$\Pi_{\textrm{next}} \leftarrow$ $\calM(s,g)$ exists} \Comment{lookup root path}
        \State $\pi_{\textrm{next}} \leftarrow$\textsc{PlanPathWithExperience}($s,g,\Pi_{\textrm{next}}$)
        \State $\pi \leftarrow$ \textsc{MergePaths}($\pi_{\textrm{curr}},\pi_{\textrm{next}},s$)
        \State \textbf{return} $\pi$
        \label{alg:query:line:c2b}
    \EndIf
%
%
    \textcolor{blue}{
    \If {$\pi \leftarrow $\textsc{PlanPathByLatching}($s,g,\pi_{\textrm{curr}}$) \textbf{successful}}
        \State \textbf{return} $\pi$
    \EndIf
    }

\EndFor
    \State \textbf{return failure} \Comment{goal is not reachable}
\EndProcedure
\end{algorithmic}
\end{algorithm}

%

\subsubsection{Latching: Reusing Root Paths}
\label{latching}
We introduce an additional step called ``Latching" to minimise the number of root paths computed in Alg.~\ref{alg:preprocess}. With latching, the algorithm tries to reuse previously-computed root paths as much as possible using special motion primitives that allow transitions from one root path to another.

For the time-configuration motion planner, the primitive is computed from a state $s \in \Pi_i$ to $s' \in \Pi_j$ such that $t(s') = t(s) + \delta_t$ by simple linear interpolation while ensuring the feasibility of the motion. Specifically, given the nominal joint velocities of the robot, if $s'$ can be reached from $s$ in time $\delta_t$, while respecting the kinematic and collision constraints, then the transition is allowed.

Clearly, this approach does not ensure that the torque limits are respected. Therefore, for the kinodynamic planner, we interpolate by fitting a cubic polynomial from the state $s \in \Pi_i$ to $s' \in \Pi_j$ that satisfies the boundary conditions. We then ensure motion validity by checking the joint velocity and torque limits of each state along the interpolated trajectory.

In Alg.~\ref{alg:preprocess}, before calling the \textsc{Preprocess} procedure for a state, the algorithm removes the set of goals that can be covered via latching, thereby reducing the number of goals that need to be covered by the \textsc{Preprocess} procedure. Correspondingly, in Alg.~\ref{alg:query}, an additional procedure is called to check if the path can be found via latching. These additions in the two pseudocodes are shown in blue. An iteration of the complete algorithm with latching is illustrated in Fig~\ref{fig:pl_latching}.

\subsection{Theoretical Analysis}

\subsubsection{CTMP Properties}
\hfill
\begin{lemma}
\label{lemma:bounded_time}
For any state \Sstart of the robot during execution, provided $t(\Sstart) \leq \Trc$, the algorithm is CTMP.
\end{lemma}

\begin{proof}
    We can prove it by showing that the query stage (Alg.~\ref{alg:query}) returns within \Tbound time and has a constant-time complexity.
    The number of times the algorithm queries~$\calM$, which is an $O(1)$ operation assuming perfect hashing, is bounded by $O(l)$ where $l=\Trc/\delta_t$ is the maximum number of time steps from $t = 0$ to $\Trc$. 
    The number of times the algorithm attempts to latch on to a root path (namely, a call to \textsc{CanLatch}  which is a constant-time operation) is also bounded by $l$. 
    As $l$ is constant for of a fixed time cut off \Trc,  the execution time of the aforementioned operations constitute \Tconst (constant-time).
    Finally, Alg.~\ref{alg:query} calls the \textsc{Plan} method only once.
    Since the execution time of \textsc{Plan} is bounded by \Tbound - \Tconst, the overall execution time of Alg.~\ref{alg:query} is \Tbound. Hence it is a CTMP algorithm.
\end{proof}

\begin{lemma}
For any state \Sstart of the robot during execution, provided $t(\Sstart) \leq \Trc$, the algorithm is CTMP-complete.
\end{lemma}

\begin{proof}[Proof]
In order to prove it we need to show that for any \Sstart that the system can be at, (1) if a~$g$ is reachable from \Sstart, it is \emph{covered} by \Sstart in Alg.~\ref{alg:preprocess} and (2) if~$g$ is covered by \Sstart, Alg.~\ref{alg:query} is guaranteed to return a path within \Tbound time.

Alg.~\ref{alg:preprocess} starts by computing a set of root paths from \Shome that ensures that it covers all of its reachable goals. It then iterates over all states on these paths and adds additional root paths ensuring that these states also cover their reachable goals. It does it recursively until no state \Sstart before \Trc is left with any uncovered $g$ which could be reachable.
    
Alg.~\ref{alg:preprocess} covers~$g$ via at least one state between \Sstart and the state at~$t_{rc}$ (inclusively) (loop at line~\ref{loop2}).
In query phase, Alg.~\ref{alg:query} iterates through all states between \Sstart and the state at~$t_{rc}$ (inclusively) to identify the one that covers~$g$ (loop at line~\ref{alg:query:line:c2a}). Since~$g$ is covered by at least one of these states by Alg.~\ref{alg:preprocess}, Alg.~\ref{alg:query} is guaranteed to find a path from \Sstart to $g$.
Moreover, from lemma~\ref{lemma:bounded_time} we have that the path is returned within \Tbound time. Hence the algorithm is CTMP-complete.

\end{proof}

\subsubsection{Space Complexity}

\vspace{2mm}
\begin{lemma}
    Let $n_\Pi$ be the maximum number of root paths needed to cover a goal region for a given state \Sstart and $\ell'$ be the number of discretized time steps at which the algorithm computes root paths, then algorithm requires~$O(n_\Pi^{\ell'})$ space.
\end{lemma}
\begin{proof}
 At first, the algorithm stores~$n_\Pi$ paths~(worst case) from \Shome to \Gfull consuming~$O(n_\Pi)$ space. It then computes root paths starting with the states at \Trc and iterating backwards to previous time steps~(see Alg.~\ref{alg:preprocess} line~\ref{loop2}). The loop terminates when no more uncovered goals are left. If~$\ell' - 1$ is the maximum number of iterations of loop at line~\ref{loop2} that Alg.~\ref{alg:preprocess} undergoes, then it requires ~$O(n_\Pi^{\ell'})$ space.
\end{proof}

Note that Alg.~\ref{alg:preprocess} reduces the space complexity of the naive approach~(Sec.~\ref{subsec:strawman}) primarily by (1) compressing the number of paths required to cover a goal region from~$n_{\rm goal}$ to $n_\Pi$ by planning with experience reuse and (2) compressing the number of time steps for which the algorithm requires preprocessing for from $\ell$ to $\ell'$ by leveraging~\ref{obs:2}.
Additionally, the latching feature allows the algorithm to further reduce the required number of paths by a significant amount.

Pertaining to the structure of our problem, the algorithm works under the assumption that $n_\Pi \ll n_{\text{goal}}$ and $\ell' \ll \ell$.  To demonstrate the magnitude of compression, in preprocessing for the time-configuration planner (see Table.~\ref{table:pp1}), the algorithm covers 7197 goals with only nine root paths for \Shome and only computes root paths for a single time step $t = 0$ (i.e., $\ell' = 1$) out of a total of eight time steps up to \Trc. 

\section{Evaluation}
\label{sec:eval}
We evaluated our algorithm in simulation and on a real robot. The conveyor speed that we used for all of our results was 0.2$m/s$. 
We used Willow Garage's PR2 robot in our experiments using its 7-DOF arm. For the real robot experiments we tested with the time-configuration planner whereas for the simulated results we tested both for the time-configuration and kinodynamic planners.
The additional time dimension made the planning problem eight dimensional.
The experiments can be viewed at
~\url{https://youtu.be/iLVPBWxa5b8}.

\begin{table*}[t]
\begin{subtable}[h]{\textwidth}
\centering
\begin{adjustbox}{width=1\textwidth}
\begin{tabular}{|c||c||c|c|c||c|c|c||c|c|c|}
\hline
   & CTMP 
   & \multicolumn{3}{c|}{wA*}
   & \multicolumn{3}{c|}{E-Graph}
   & \multicolumn{3}{c|}{RRT}
   \\ \hline
   & $T_{b}$ = \textbf{0.2} 
   & $T_{b}$ = 0.5 & $T_{b}$ = 1.0 & $T_{b}$ = 2.0 
   & $T_{b}$ = 0.5 & $T_{b}$ = 1.0 & $T_{b}$ = 2.0 
   & $T_{b}$ = 0.5 & $T_{b}$ = 1.0 & $T_{b}$ = 2.0 
   \\ \hline
Pickup success [\%]                   
& \textbf{100} & 0.0 & 0.0 & 18.0 & 0.0 & 0.0 & 80.0 & 0.0 & 0.0 & 18.0 \\ \hline
Planning success [\%]                  
& \textbf{100} & 4.0 & 17.0 & 19.0 & 31.0 & 80.0 & 90.0 & 12.0 & 9.0 & 13.0 \\ \hline
Planning time [s]
& \textbf{0.085} & 0.433 & 0.628 & 0.824 & 0.283 & 0.419 & 0.311 & 0.279 & 0.252& 0.197\\ \hline
Planning cycles 
& \textbf{3} & 2 & 2 & 2 & 2 & 2 & 2 & 2 & 2 & 2 \\ \hline
Path cost [s]                         & 9.49        & 8.19          & 8.28          & \textbf{7.60}          & 8.54          & 8.22          & 7.90          & 9.68          & 8.96          & 8.04          \\ \hline
\end{tabular}
\end{adjustbox}
\caption{Simulation results---Time-configuration Planner.}
\label{tab:sim_results_p1}
\end{subtable}

\begin{subtable}[h]{\textwidth}
\centering
\begin{adjustbox}{width=0.75\textwidth}
\begin{tabular}{|c||c||c|c|c||c|c|c|}
\hline
   & CTMP 
   & \multicolumn{3}{c|}{wA*}
   & \multicolumn{3}{c|}{E-Graph}
   \\ \hline
   & $T_{b}$ = \textbf{0.25} 
   & $T_{b}$ = 0.5 & $T_{b}$ = 1.0 & $T_{b}$ = 2.0 
   & $T_{b}$ = 0.5 & $T_{b}$ = 1.0 & $T_{b}$ = 2.0 
   \\ \hline
Pickup success [\%]                   
& \textbf{100} & 0.0 & 0.0 & 18.0 & 0.0 & 0.0 & 24.0 \\ \hline
Planning success [\%]                  
& \textbf{100} & 14.0 & 27.0 & 44.0 & 35.0 & 22.0 & 38.0 \\ \hline
Planning time [s]
& \textbf{0.122} & 0.325 & 0.732 & 1.45 & 0.1593 & 0.405 & 0.413\\ \hline
Planning cycles 
& \textbf{3} & 2 & 2 & 2 & 2 & 2 & 2 \\ \hline
Path cost [s]                         & 9.82        & 8.74          & 9.15          & 9.39          & 8.63          & 8.72          & \textbf{8.15}          \\ \hline
\end{tabular}
\end{adjustbox}
\caption{Simulation results---Kinodynamic Planner.}
\label{tab:sim_results_p2}
\end{subtable}
\caption{Simulation results averaged over 50 experiments. Here $T_b$ denotes the (possibly arbitrary) time bound that the algorithm uses. Note that for our method $T_b = \Tbound$ is the time bound that the algorithm is ensured to compute a plan.
CTMP algorithm uses a fixed $T_b$ (0.2$s$). For the baselines, since their planning times are unbounded, we test them with $T_b =$ 0.5, 1.0 and 2.0$s$. Our CTMP approach shows a task and planning success rate of 100$\%$ for both the planners. Among the baselines, E-Graph shows the best performance, however its performance drops significantly for the kinodynamic planning problem.}
\label{tab:sim_results}
\end{table*}



\subsection{Experimental setup}

\subsubsection{Sense-plan-act cycle}
As object $o$ moves along the conveyor belt, we use the Brute Force ICP pose estimation baseline proposed in \cite{perch} to obtain its 3-Dof pose for each captured input point cloud.
We follow the classical sense-plan-act cycle as depicted in Fig.~\ref{fig:tl}.
Specifically, 
the perception system captures an image (point cloud) of the object~$o$ at time~$t_{\textrm{img}}$
followed by a period of duration $T_{\textrm{perception}}$ in which the perception system estimates the pose of $o$.
At time $t_{\textrm{msg}} = t_{\textrm{img}} + T_{\textrm{perception}}$, planning starts for a period of $T_{\textrm{planning}}$ which is guaranteed to be less than~$\Tbound$.
Thus, at $t_{\textrm{plan}} = t_{\textrm{msg}} + T_{\rm planning}$ the planner waits for an additional duration of $T_{\rm wait} = \Tbound - T_{\rm planning}$.
Finally, at~$t_{\textrm{exec}} = t_{\textrm{plan}} + T_{\rm wait}$, the robot starts executing the plan. Note that the goal $g$ that the planner plans for is not for the object pose at $t_{\textrm{img}}$ but its forward projection in time to $t_{\textrm{exec}}$ to account for $T_{\textrm{perception}}$ and $T_{\textrm{bound}}$.
While executing the plan, if we obtain an updated pose estimate, the execution is preempted and the cycle repeats.

\begin{figure}[t]
    \centering
    \includegraphics[width=0.45\textwidth]{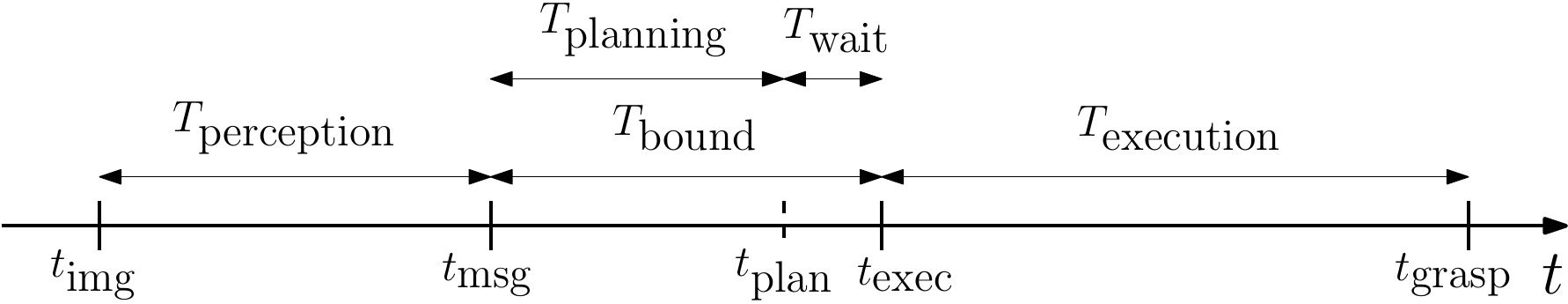}
    \caption{\CaptionTextSize Timeline of the sense-plan-act cycle.}
    \label{fig:tl}
\end{figure}

\subsubsection{Goal region specification}
To define the set of all goal poses~$\Gfull$, we need to detail our system setup, depicted in Fig.~\ref{fig:pe}.
The conveyor belt moves along the $x$-axis from left to right.
We pick a fixed $x$-value termed~\Xexec, such that when the incoming $o$ reaches \Xexec as per the perception information, at that point we start execution.

%
Recall that a pose of an object~$o$ is a three dimensional point~$(x,y,\theta)$ corresponding to the $(x,y)$ location of~$o$ and to its orientation (yaw angle) along the conveyor belt.
$\Gfull$ contains a fine discretization of all possible $x,y$ and $\theta$ values in  $[\Xexec-2\varepsilon, \Xexec+2\varepsilon]$.
We select $\Gfull$ such that $\varepsilon = $2.5$cm$, making the goal region 10$cm$ long along x-axis. Its dimension along $y$-axis is 20$cm$, equal to the width of the conveyor belt. The discretization in $x,y$ and $\theta$ is 1.0$cm$ and 10$\deg$ respectively.

In the example depicted in Fig.~\ref{fig:pe}, the thick and the thin solid rectangles show the ground truth and estimated poses, respectively at two time instances in the life time of the object.
The first plan is generated for the pose shown at $x_{\textrm{exec}}$. During execution, the robot receives an improved estimate and has to replan for it. At this point we back project this new estimate in time using the known speed of the conveyor and the time duration between the two estimates. This back-projected pose (shown as the dotted rectangle) is then picked as the new goal for replanning. Recall that under our assumption about the pose error in perception being $\varepsilon$, the back projected pose will always lie inside \Gfull.
%

\begin{figure}
    \centering
    \includegraphics[width=0.48\textwidth]{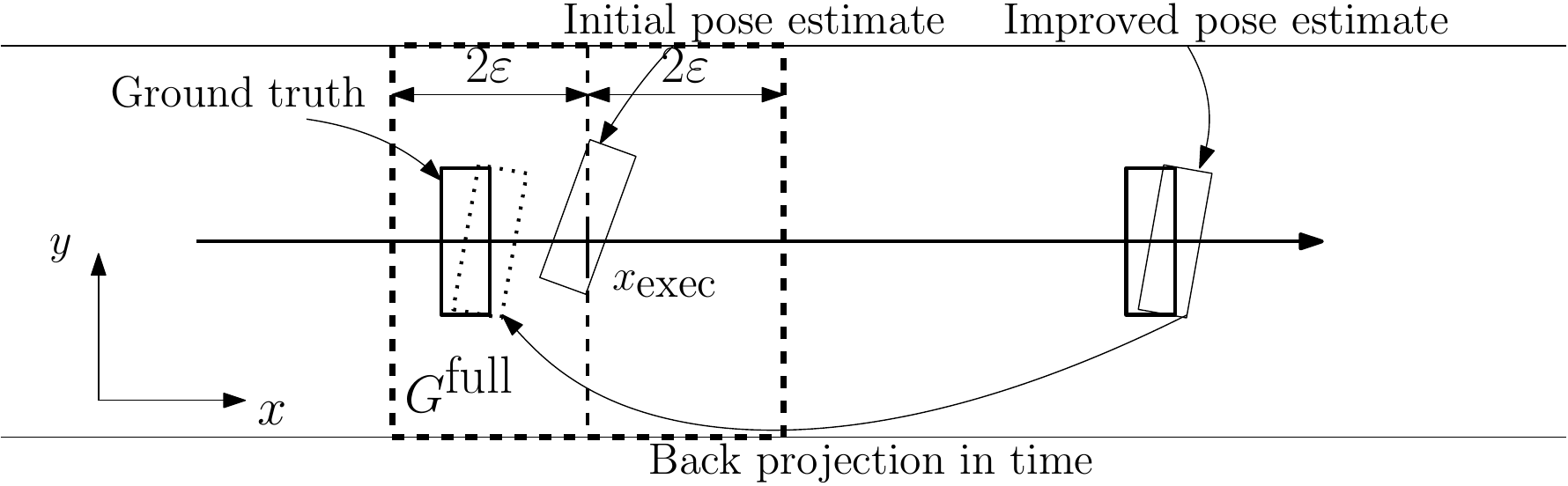}
    \caption{\CaptionTextSize A depiction of \Gfull-specification on a conveyor belt (overhead view) and perception noise handling. 
    }
    \label{fig:pe}
\end{figure}

\subsection{Results}

\begin{table*}[t]
\begin{subtable}[h]{\textwidth}
\begin{adjustbox}{width=1\textwidth}
\begin{tabular}{|c||c|c|c|c|c|c|c|c|c|c|}
\hline
State time & \begin{tabular}[c]{@{}c@{}}Num. of\\ states\end{tabular} & \begin{tabular}[c]{@{}c@{}}Unreachable\\ goals\end{tabular} & \begin{tabular}[c]{@{}c@{}}Covered\\ goals\end{tabular} & \begin{tabular}[c]{@{}c@{}}Covered via\\ root paths\end{tabular} & \begin{tabular}[c]{@{}c@{}}Covered via\\ Latching\end{tabular} & \begin{tabular}[c]{@{}c@{}}Num. of\\ root paths\end{tabular} & \begin{tabular}[c]{@{}c@{}}Num. of\\ states\end{tabular} & \begin{tabular}[c]{@{}c@{}}Latching \\ tries\end{tabular} & \begin{tabular}[c]{@{}c@{}}Latching\\ failures\end{tabular} & \begin{tabular}[c]{@{}c@{}}Processing\\ time\end{tabular} \\ \hline
0 & 1 & 3 & 7197 & 7197 & 0 & 9 & 1641 & 0 & 0 & 2534 \\ \hline
3.5 & 9 & 3 & 7197 & 0 & 7197 & 0 & 0 & 8 & 0 & 0.01 \\ \hline
\end{tabular}%
\end{adjustbox}
\caption{Preprocessing statistics---Time-configuration Planner}
\label{table:pp1}
\end{subtable}

\begin{subtable}[h]{\textwidth}
\begin{adjustbox}{width=1\textwidth}
\begin{tabular}{|c||c|c|c|c|c|c|c|c|c|c|}
\hline
State time & \begin{tabular}[c]{@{}c@{}}Num. of\\ states\end{tabular} & \begin{tabular}[c]{@{}c@{}}Unreachable\\ goals\end{tabular} & \begin{tabular}[c]{@{}c@{}}Covered\\ goals\end{tabular} & \begin{tabular}[c]{@{}c@{}}Covered via\\ root paths\end{tabular} & \begin{tabular}[c]{@{}c@{}}Covered via\\ Latching\end{tabular} & \begin{tabular}[c]{@{}c@{}}Num. of\\ root paths\end{tabular} & \begin{tabular}[c]{@{}c@{}}Num. of\\ states\end{tabular} & \begin{tabular}[c]{@{}c@{}}Latching \\ tries\end{tabular} & \begin{tabular}[c]{@{}c@{}}Latching\\ failures\end{tabular} & \begin{tabular}[c]{@{}c@{}}Processing\\ time\end{tabular} \\ \hline
0 & 1 & 16 & 7184 & 7184 & 0 & 18 & 3120 & 0 & 0 & 25237.7 \\ \hline
3.0 & 16 & 16 & 7184 & 0 & 7184 & 0 & 0 & 13.2 & 0 & 1062 \\ \hline
3.5 & 18 & 38.2 & 7161.8 & 444.6 & 6717.2 & 11.7 & 2203.4 & 18 & 2.4 & 0.02 \\ \hline
\end{tabular}%
\end{adjustbox}
\caption{Preprocessing statistics---Kinodynamic Planner}
\label{table:pp2}
\end{subtable}
\caption{Preprocessing statistics--- We report the preprocessing statistics for the robot states from which the algorithm either computes a latching primitive or computes new root paths for replanning. The statistics are indexed based on the time step of the states. For each time step, we give the number of states at that time step and the average values for each preprocessing metric over all the states at that time stamp.}
\label{table:pp}
\end{table*}

\subsubsection{Real-robot experiments}
\label{sec:robot_results}
To show the necessity of real-time replanning in response to perception updates, we performed three types of experiments: 
using our approach to replan every time new object pose estimate arrives, single-shot planning based on the first object pose estimate 
and single-shot planning using the late (more accurate) pose estimate. 
For each set of experiments, we determined the pickup success rate to grasp the moving object (sugar box) off the conveyor belt. In addition, we report on the perception system's success rate by observing the overlap between the point cloud of the object's 3D model transformed by the predicted pose (that was used for planning) and the filtered input point cloud containing points belonging to the object. 
A high (low) overlap corresponds to an accurate (inaccurate) pose estimate. 
We use the same strategy to determine the range for which the perception system's estimates are accurate and use it to determine the time for the best-pose planning. 
Further, for each method, we determine the pickup success rate given that the perception system's estimate was or was not accurate. 

The experimental results are shown in Table \ref{tab:robot_results}. Our method achieves the highest overall pickup success rate on the robot by a large margin, indicating the importance of continuous replanning with multiple pose estimates. 
First-pose planning has the least overall success rate due to inaccuracy of pose estimates when the object is far from the robot's camera. 
Best-pose planning performs better overall than the first pose strategy, since it uses accurate pose estimates, received when the object is close to the robot. However it often fails even when perception is accurate, since a large number of goals are unreachable due to limited time remaining to grasp the object when it is closer to the robot.
A demonstration of our approach is given in Fig.~\ref{fig:demo}.

\begin{table*}[t]
\centering
\begin{tabular}{|c|c|c|c|c|}
\hline
                    & \begin{tabular}[c]{@{}c@{}}Success \\ rate\end{tabular}
                    & \begin{tabular}[c]{@{}c@{}}Accuracy of \\ perception $[\%]$ \end{tabular} 
                    & \begin{tabular}[c]{@{}c@{}}Success rate \\ (Accurate perception)\end{tabular} 
                    & \begin{tabular}[c]{@{}c@{}}Success rate \\ (Inaccurate perception)\end{tabular} \\ \hline
CTMP          & \textbf{69.23}                                                                      & 42.31                   & \textbf{83.33}                                                                             & \textbf{57.14}                                                                             \\ \hline
First-pose Planning & 16.00                                                                      & 24.00                   & 66.67                                                                             & 0.00                                                                              \\ \hline
Best-pose Planning   & 34.61                                                                      & 34.62                   & 55.56                                                                             & 23.53                                                                             \\ \hline
\end{tabular}
\caption{\CaptionTextSize Real-robot experiments. Success rate for the three experiments (Our method (CTMP), First-pose planning and Best-pose planning) averaged over 50 trials .}
\label{tab:robot_results}
\end{table*}

\ignore{
\begin{table*}[t]
\centering
\begin{tabular}{|c|c|c|c|c|}
\hline
                    & \begin{tabular}[c]{@{}c@{}}Pickup success rate [\%]\\ (Overall)\end{tabular} & Perception success rate [\%]& \begin{tabular}[c]{@{}c@{}}Pickup success rate [\%]\\ (Perception = 1*)\end{tabular} & \begin{tabular}[c]{@{}c@{}}Pickup success rate [\%]\\ (Perception = 0*)\end{tabular} \\ \hline
Our method (E1)          & \textbf{9.23}                                                                      & \textbf{42.31}                   & \textbf{83.33}                                                                             & \textbf{57.14}                                                                             \\ \hline
First-pose planning (E2) & 16.00                                                                      & 24.00                   & 66.67                                                                             & 0.00                                                                              \\ \hline
Best-pose planing (E3)   & 34.61                                                                      & 34.62                   & 55.56                                                                             & 23.53                                                                             \\ \hline
\end{tabular}
\caption{Real-robot Experiments}
\label{tab:robot_results}
\end{table*}
}

\begin{figure*}[t]
    \centering
    \begin{subfigure}{.31\textwidth}
        \includegraphics[trim=0 0 400 0, clip, width=\textwidth]{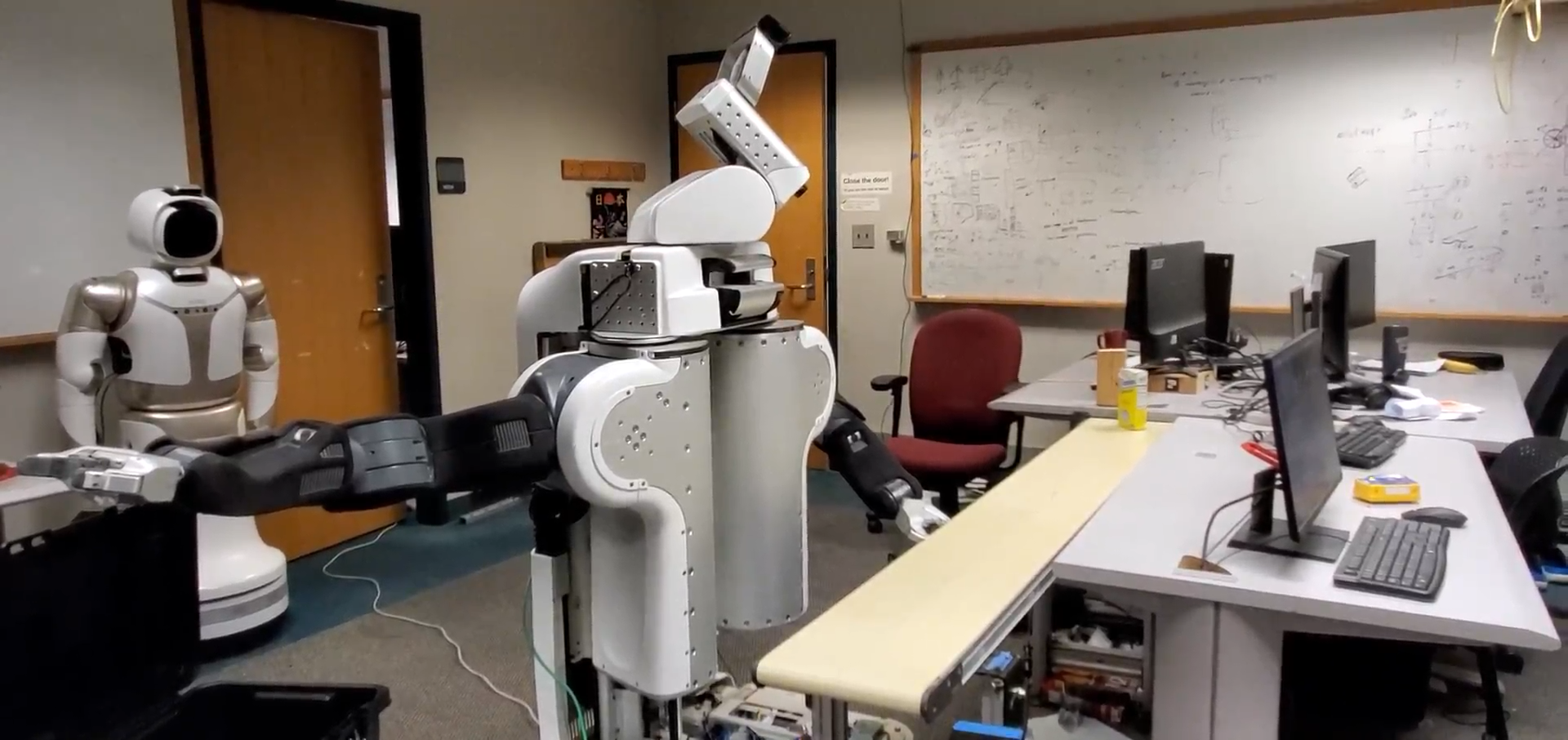}
        \caption{}
        \label{fig:demo1}
    \end{subfigure}
    \hspace{2mm}
    \begin{subfigure}{0.31\textwidth}
        \includegraphics[trim=0 0 400 0, clip, width=\textwidth]{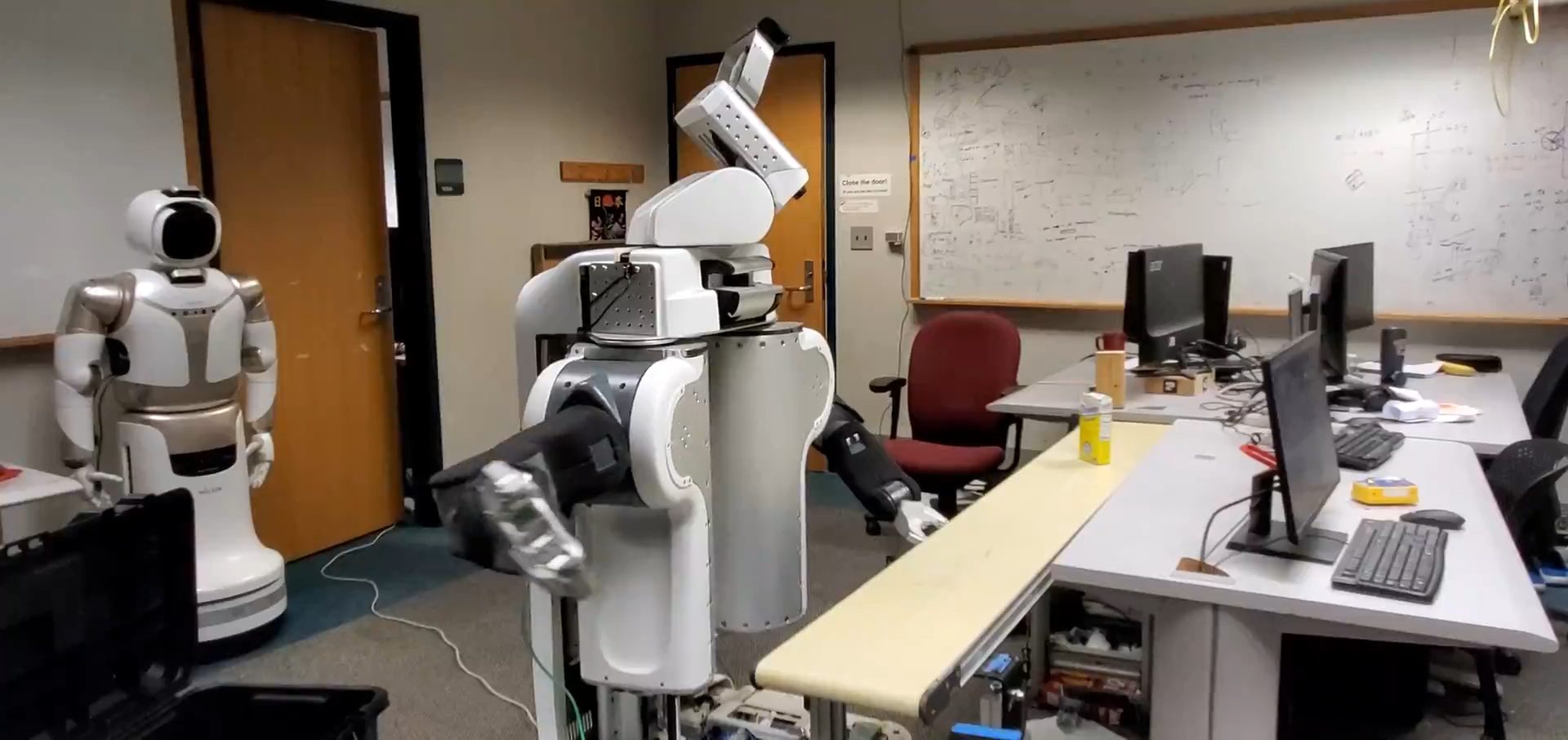}
        \caption{}
        \label{fig:demo2}
    \end{subfigure}
    \hspace{2mm}
    \begin{subfigure}{0.31\textwidth}
        \includegraphics[trim=0 0 400 0, clip, width=\textwidth]{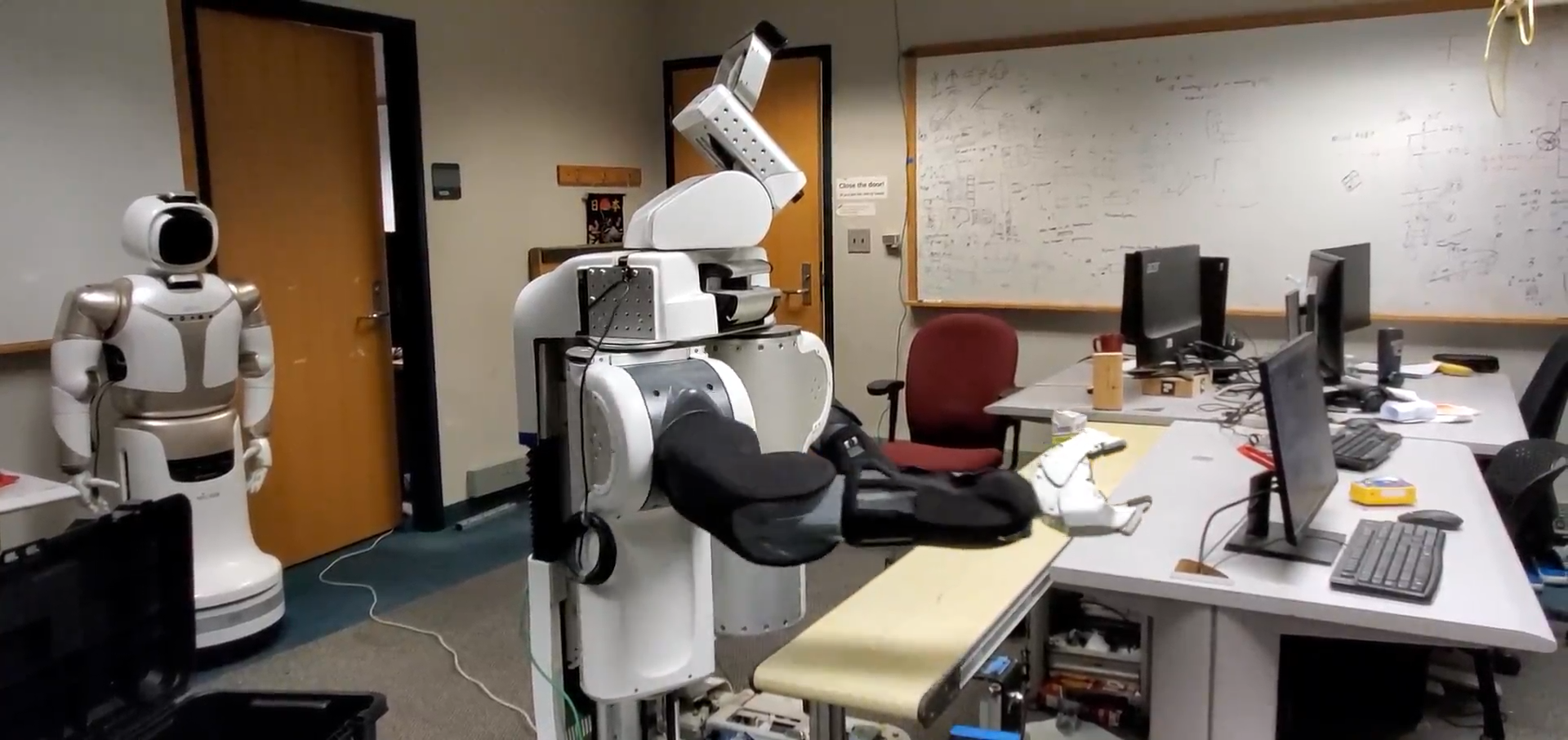}
        \caption{}
        \label{fig:demo3}
    \end{subfigure}
    \begin{subfigure}{0.31\textwidth}
        \includegraphics[trim=0 0 400 0, clip, width=\textwidth]{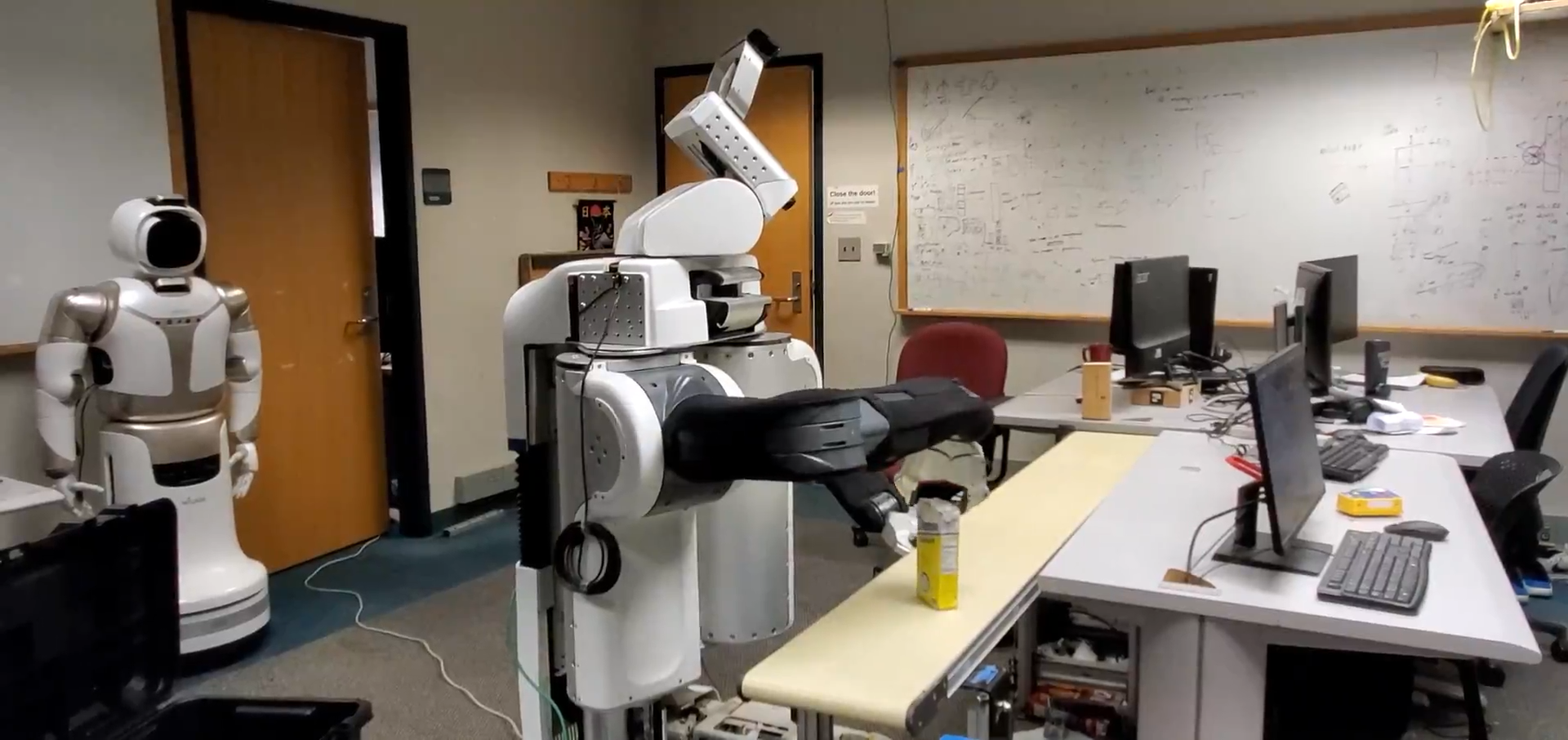}
        \caption{}
        \label{fig:demo4}
    \end{subfigure}
    \hspace{2mm}
    \begin{subfigure}{0.31\textwidth}
        \includegraphics[trim=0 0 400 0, clip, width=\textwidth]{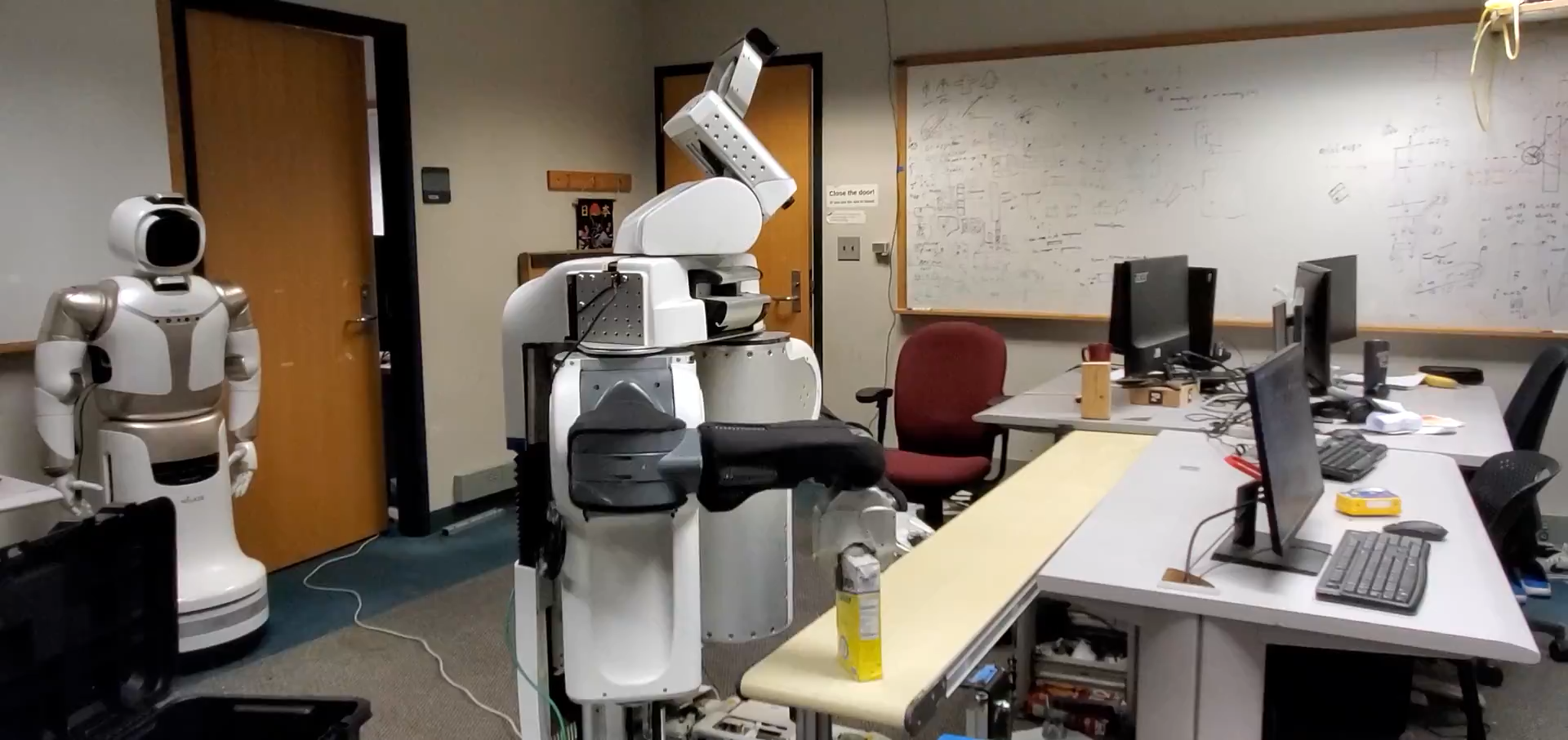}
        \caption{}
        \label{fig:demo5}
    \end{subfigure}
    \hspace{2mm}
    \begin{subfigure}{0.31\textwidth}
        \includegraphics[trim=0 0 400 0, clip, width=\textwidth]{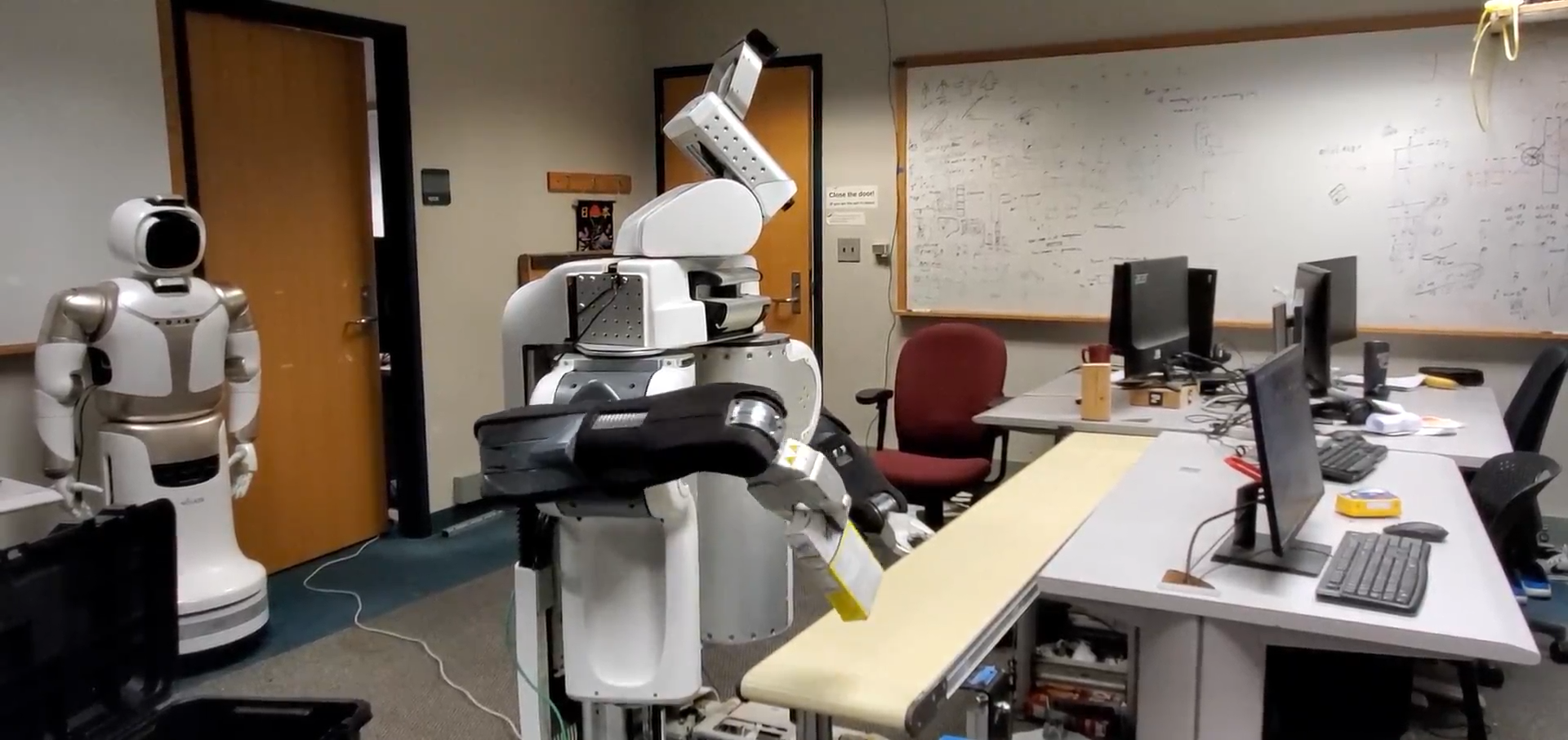}
        \caption{}
        \label{fig:demo6}
    \end{subfigure}
    \caption{\CaptionTextSize
    Snapshots from a real robot experiment.
    (\subref{fig:demo1})~The robot receives the first pose estimate from the perception system, generates the first plan within \Tbound and starts execution.
    (\subref{fig:demo2})~The robot receives the second pose estimate with a pose correction of distance~3cm and replans to the new goal
    (\subref{fig:demo3})~The robot receives the third and last pose estimate with no further correction and hence continues to follow the previous plan.
    (\subref{fig:demo4}) and (\subref{fig:demo5}) The robot executes the dynamic motion primitive to reach the grasp pose and glide along with the object.
    (\subref{fig:demo6}) The robot lifts up the sugar box from the conveyor.
    }
    \label{fig:demo}
\end{figure*}

\subsubsection{Simulation experiments}
We simulated the real world scenario to evaluate our method against other baselines. We compared our method with wA*~\cite{pohl1970heuristic}, E-graph~\cite{PCCL12} and RRT~\cite{lavalle1998rapidly}. 
For wA* and E-graph we use the same graph representation as our method. 
For E-graph we precompute five paths to randomly-selected  goals in \Gfull. 
We adapt the RRT algorithm to account for the under-defined goals. To do so, we sample pre-grasp poses along the conveyor 
and compute IK solutions for them to get a set of goal configurations for goal biasing. 
When a newly-added node falls within a threshold distance from the object, we use the same dynamic primitive that we use in the search-based methods to add the final grasping maneuver. If the primitive succeeds, we return success. We also allow wait actions at the pre-grasp locations.

For any planner to be used in our system, we need to endow it with a (possibly arbitrary) planning time bound to compute the future location of the object from which the new execution will start.
%
If the planner fails to generate the plan within this time, the robot fails to react to that pose update and such cases are recorded as planning failure. 
We label a run as a pickup success if the planner successfully replans once after the object crosses the 1.0m mark. The mark is the mean of accurate perception range that was determined experimentally and used in the robot experiments as described in Section \ref{sec:robot_results}.
The key takeaway from our experiments (Table~\ref{tab:sim_results}) is that having a known time bound on the query time is vital to the success of the conveyor pickup task.

Our method shows the highest pickup success rate, planning success rate (success rate over all planning queries) and an order of magnitude lower planning times compared to the other methods. 
We tested the other methods with several different time bounds. After our approach E-graph performed decently well. RRT suffers from the fact that the goal is under-defined and sampling based planners typically require a goal bias in the configuration space. Another important highlight of the experiments is the number of planning cycles over the lifetime of an object. While the other approaches could replan at most twice, our method was able to replan thrice due to fast planning times.

\subsubsection{Preprocessing}
The statistics of the preprocessing stage (i.e. running Alg.~\ref{alg:preprocess}) are shown in Table~\ref{table:pp}. The offline time bound $T_{\calP}$ used in all of our experiments was 10$s$
%
%
%
In all experiments, we used $\Trc =$3.5$s$ and $\delta_t = $0.5$s$.
For the time-configuration planner the preprocessing took 2,534$s$. Only nine root paths are computed to cover 7,197 goals (three goals being unreachable and hence uncovered). For the states at $t=\Trc ($3.5$s)$, there were no latching failures and therefore, the algorithm terminates without preprocessing for earlier time stamps.
For the kinodynamic planner, the preprocessing takes about 7 hours. The dynamic constraints causes latching failures and therefore, the algorithm requires more preprocessing efforts. Due to latching failures at $t=$3.5, it needs to compute root paths for some of the states at this time step. Finally, it covers all the uncovered goals for states at $t=$3.0 via latching and finishes preprocessing.

\section{Conclusion and Discussion}
To summarize, we developed a provably constant-time planning and replanning algorithm that can be used to grasp fast moving objects off conveyor belts and evaluated it in simulation and in the real world on the PR2 robot. Through this work, we advocate the need for algorithms that guarantee (small) constant-time planning for time-critical applications, such as the conveyor pickup task, which are often encountered in warehouse and manufacturing environments. To this end we introduce and formalize a new class of algorithms called CTMP algorithms.

An interesting future research direction could be to leverage roadmap-based representation instead of storing individual paths in a way that the algorithm remains CTMP-complete, namely it maintains constant-time planning guarantee.
On the more practical side, a useful extension could be to parallelize the preprocessing computation over multiple CPU cores. One obvious way of doing so is within the Alg.~\ref{alg:step1}. The required uncovered goal region \Guncov can be divided over multiple threads.

Another useful extension to our CTMP approach is to be able to handle multiple objects simultaneously coming on the conveyor belt. This setting is common for a sorting task when multiple robot arms work at the same conveyor and they have to pick up one object while avoiding collisions with the other object. This makes the problem more challenging since the planner has to consider more than one dynamic objects in the scene.

\begin{acks}
This work was supported by the ONR grant N00014-18-1-2775 and the ARL grant W911NF-18-2-0218 as part of the A2I2 program.
In addition, it was partially supported by
the Israeli Ministry of Science \& Technology grant No. 102583 and by Grant No. 1018193 from the United States-Israel Binational Science Foundation (BSF) and by Grant No. 2019703 from the United States National Science Foundation (NSF).

The authors would like to thank Andrew Dornbush for providing support for the Search-Based Motion Planning Library~(\url{https://github.com/aurone/smpl}) which is used in our implementation. The authors would also like to thank Ellis Ratner for fruitful discussions and Anirudh Vemula for helping out in restoring the PR2 robot which is used in our experiments.

\end{acks}

\balance
\bibliographystyle{SageH}
\bibliography{main.bib}

\end{document}